\documentclass[journal,12pt,onecolumn,draftclsnofoot,]{IEEEtai}

\usepackage[colorlinks,urlcolor=blue,linkcolor=blue,citecolor=blue]{hyperref}

\usepackage{color,array}

\usepackage{graphicx}
\usepackage{amsmath,amssymb,amsthm}
\usepackage[ruled]{algorithm2e}
\usepackage[fixamsmath]{mathtools}
\usepackage{nicefrac}
\usepackage{cleveref}
\usepackage{subfigure}

\newcommand{\E}{\mathbb{E}}
\newcommand{\N}{\mathbb{N}}
\newcommand{\C}{\mathbb{C}}
\renewcommand{\P}{\mathbb{P}}
\newcommand{\R}{\mathbb{R}}

\newcommand{\cA}{\mathcal{A}}

\newcommand{\cF}{\mathcal{F}}

\newcommand{\cN}{\mathcal{N}}
\newcommand{\cO}{\mathcal{O}}

\newcommand{\cR}{\mathcal{R}}
\newcommand{\cS}{\mathcal{S}}

\newcommand{\cU}{\mathcal{U}}

\newcommand{\Ew}[1]{\mathbb{E}\left[#1\right]}
\renewcommand{\Pr}[1]{\mathbb P\left(#1\right)}
\newcommand{\norm}[1]{\lVert #1 \lVert}

\newtheorem{theorem}{Theorem}
\newtheorem{lemma}{Lemma}
\newtheorem{proposition}{Proposition}
\newtheorem{definition}{Definition}
\newtheorem{corollary}{Corollary}
\newtheorem{remark}{Remark}
\usepackage{todonotes}

\begin{document}

\title{3DPG: Distributed Deep Deterministic Policy Gradient Algorithms for Networked Multi-Agent Systems}

\author{Adrian Redder,  Arunselvan Ramaswamy, and Holger Karl
\thanks{A. Redder was supported by the German Research Foundation (DFG) - 315248657 and SFB 901. }
\thanks{A. Redder is with the Computer Science Department, Paderborn University (e-mail: aredder@mail.upb.de).}
\thanks{A. Ramaswamy is with the Computer Science Department, Karlstad University (e-mail: arunselvan.ramaswamy@kau.se).}
\thanks{H. Karl is with the Hasso-Plattner-Institute, University of Potsdam (e-mail: holger.karl@hpi.de).}
}


\maketitle

\begin{abstract}

 
We present Distributed Deep Deterministic Policy Gradient (3DPG), a multi-agent actor-critic (MAAC) algorithm for Markov games. Unlike previous MAAC algorithms, 3DPG is fully distributed during both training and deployment. 3DPG agents calculate local policy gradients based on the most recently available local data (states, actions) and local policies of other agents. During training, this information is exchanged using a potentially lossy and delaying communication network. The network therefore induces Age of Information (AoI) for data and policies. We prove the asymptotic convergence of 3DPG even in the presence of potentially unbounded Age of Information (AoI). This provides an important step towards practical online and distributed multi-agent learning since 3DPG does not assume information to be available deterministically. We analyze 3DPG in the presence of policy and data transfer under mild practical assumptions. Our analysis shows that 3DPG agents converge to a local Nash equilibrium of Markov games in terms of utility functions expressed as the expected value of the agents local approximate action-value functions (Q-functions). The expectations of the local Q-functions are with respect to limiting distributions over the global state-action space shaped by the agents' accumulated local experiences. Our results also shed light on the policies obtained by general MAAC algorithms. We show through a heuristic argument and numerical experiments that 3DPG improves convergence over previous MAAC algorithms that use old actions instead of old policies during training. Further, we show that 3DPG is robust to AoI; it learns competitive policies even with large AoI and low data availability.

\end{abstract}


\begin{IEEEkeywords}
Actor-Critic Algorithms, Age of Information, Asymptotic Convergence, Deep Multi-Agent Reinforcement Learning, Distributed Online Learning, Networked Systems
\end{IEEEkeywords}

\section{Introduction}
\label{sec:introduction}

\IEEEPARstart{M}{ulti}-Agent Actor-Critic (MAAC) algorithms are an important and popular class of Deep Reinforcement Learning (DeepRL) algorithms for intelligent decision making in multi-agent systems (MAS) \cite{arulkumaran2017deep}. 
MAAC algorithms, like the popular Multi-Agent Deep Deterministic Policy Gradient (MADDPG) algorithm \cite{lowe2017multi}, typically assume \emph{instant access to global data} in order to train coordinated decentralized policies: to train MADDPG, agents need instant access to all agent polices, their action sequences, and global state information. This training paradigm is called \emph{centralized training with decentralized execution}.

This central perspective can be justified only when the training involves an accurate simulator of the environment or when the agents are connected by communication network without transmission errors and delays. These assumption are often impractical. For example, in edge computational task offloading \cite{sofla2022towards}, observations and decisions are inherently local and not available globally to all edge nodes. Algorithms such as MADDPG are therefore \emph{unsuitable for such truly distributed online multi-agent learning problems}. Other works, such as \cite{zhang2018fully}, propose an algorithm for fully decentralized multi-agent learning. Here, it is assumed that the state and action spaces of the agents are finite, that the global state space can be observed by all agents, and that the resulting Markov chain is ergodic for each potential policy.  These assumptions will be hard to verify in practice. Instead, we propose an algorithm that only uses local information and we show explicitly that the limiting policies found by our algorithm result in an ergodic Markov chain.

The assumptions in the literature regarding centralized training and global data access stem from the need for coordinated learning among multiple agents. To provide a distributed online learning algorithm based on MADDPG, all local states, actions, and policies
must be made available to all agents in a synchronized and lossless manner, across an imperfect network. That is, in most cases, prohibitively expensive. Our work addresses this issue and presents the Distributed Deep Deterministic Policy Gradient (3DPG) algorithm, a natural distributed extension of the single agent DDPG algorithm \cite{lillicrap2015continuous} that enables coordinated but fully distributed online multi-agent learning. 3DPG learns distributed but coordinated policies by using global data during training that has been collected by agents using communication over an available network. Notably, the agents only use local data to take decisions. In other words, individual actions are not globally coordinated, but the agents joint average behavior is coordinated.

3DPG only requires that the agents can exchange information (local states, actions and policies) in an imperfect manner: it might be lost,
or experience significant delay, causing data to have random age once it arrives -- this is commonly  described by so-called Age of Information (AoI) random variables \cite{redder22allerton}.
\emph{We analyze 3DPG under practical sufficient conditions, in particular very weak communication assumptions for the MAS. Our modest communication assumptions even allow for potentially unbounded asymptotic growth of the AoI and make no deterministic requirements regarding data availability.}

To guarantee convergence of 3DPG, we address another problem. Despite their popularity and usefulness in many practical scenarios, the conditions under which AC and MAAC algorithms converge are not well studied -- we address this gap and present practically verifiable and sufficient conditions for DDPG, 3DPG and MADDPG to converge asymptotically. Our result is based on recent progress in understanding the convergence behavior of Deep Q-Learning \cite{ramaswamy2021deep}. Such convergence guarantees and analyses are in general difficult, even for traditional \emph{single-agent} DeepRL algorithms. In the single-agent case,  RL algorithms with \emph{linear} function approximations are well studied, but algorithms that use \emph{non-linear} function approximators like Deep Neural Networks (DNNs) are not well understood. At best, convergence is only characterized under strict assumptions that are difficult to verify in practice, e.g., \cite{wang2019neural} assumes that the state transition kernel of the Markov decision process is regular; this questions the practical usefulness of such algorithms. The behavior of \emph{multi-agent} DeepRL algorithms is even more challenging since the various agents' training processes are intertwined. It is thus pertinent, both from a theoretical and from a practical standpoint, to analyze, under practical assumptions, the asymptotic properties of multi-agent DeepRL algorithms. 

\subsection{The 3DPG Algorithm}

To get a first idea (Section~\ref{sec: algo} provides details), let us view 3DPG as a multi-agent version of DDPG \cite{lillicrap2015continuous}, the most popular AC DeepRL algorithm. DDPG involves two DNNs, an actor (policy) network and a critic (Q-function) network. The actor network is trained to approximate the optimal policy, and the critic network is trained to approximate an objective function. More specifically, the critic network is trained to minimize a variant of the squared Bellman error, while the actor network is trained to pick actions that maximize the approximation of the optimal Q-function, as found by the critic. Notably, both the critic and actor networks are trained simultaneously.

In 3DPG, each agent only has access to a locally observable state (a part of the global state), can exchange information with other agents in an imperfect manner, and takes actions that affect both its local state and states observable by other agents. To take actions, each agent uses a local actor/policy. To train its local policy, each agent uses a local critic/Q-function approximation.
The local policies are functions of the local agent states, which in turn constitute the global (multi-agent) system state. 
At every discrete time step, after all agents take actions, they obtain local rewards. Each agent then uses a local copy of DDPG (as explained above) to train a local actor, while simultaneously training a local critic with respect to (w.r.t.) the global decision making problem associated with its local reward structure. 

To perform the actor training step, the agents use local \emph{policies} from other agents transferred via a potentially imperfect communication network that, e.g., causes delays. At each agent, the training step may be viewed as calculating a local policy gradient using the best available approximation of the current global policy. The 3DPG architecture at agent~$1$ of a $D$ agent system is illustrated in \Cref{fig:3DPG_flow}. In addition to the old policies of other agents, all local actor and critic training steps use data (states, actions) of other agents transferred via the communication network. This gives a quasi-centralized view but based on information with potentially large age; this view facilitates online training and execution of 3DPG. 
\begin{figure}[t]
\centering
\includegraphics[width=.485 \textwidth]{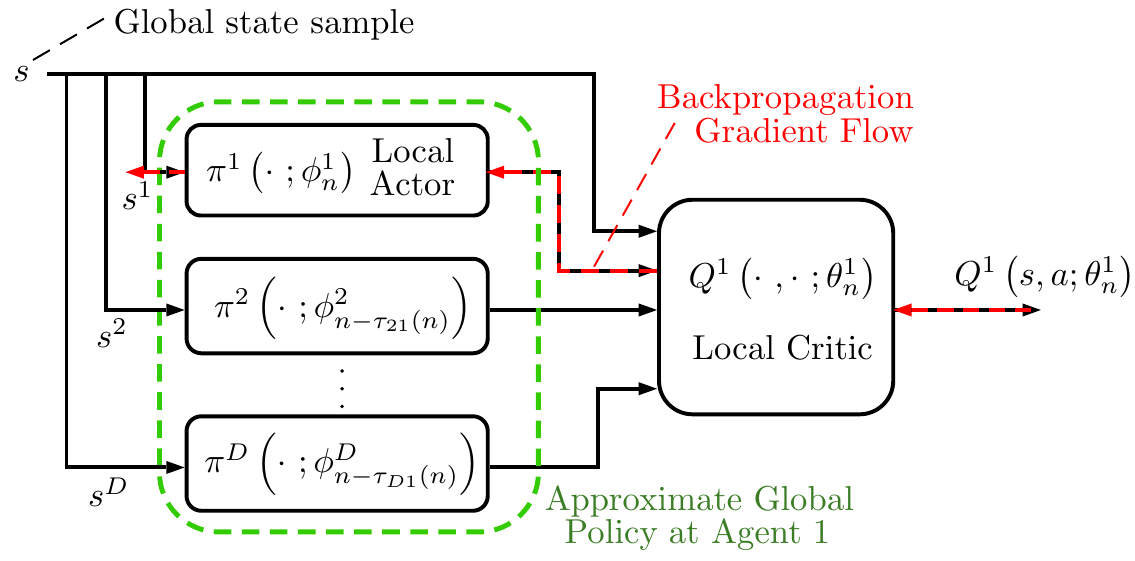}
\caption{Illustration of the 3DPG architecture at agent~$1$. The local critic is evaluated for action $ a= \pi(s;\phi_{\tau^1(n)})$ of the local approximation of the global policy. See \Cref{sec: algo} for details.}
\label{fig:3DPG_flow}
\end{figure}

The aforementioned actor training step is in stark contrast to the corresponding one in MADDPG from \cite{lowe2017multi}. Here the actor training step is based on past \emph{actions} of other agents. Since these actions are sampled from old transitions, they sometimes include random actions due to exploration. Theses actions, however, do not represent actual agent behavior and can thus negatively impact the policies found by the algorithm (see \Cref{subsec:MAACgradients,sec:comparison} for details). Using old policies of other agents, as in 3DPG, may initially increase the training variance. However, we show that for 3DPG the effect of unrepresentative behavior from old policies as well as the effect of potentially random actions does not effect the policies found by 3DPG. For MADDPG, we show that the presence of randomly explored actions can have a negative impact on the training result of MADDPG. This result is supported by numerical examples that show that 3DPG obtains better policies compared to MADDPG for problems that require coordinated decision making. 


3DPG converges even when the AoI associated with the used information of other agents during training merely has finite mean. In other words, the AoI may vary with infinite variance and may grow unbounded. The convergence in the presence of such potentially significantly outdated information (the local states, actions and policies from other agents) is achieved by agents using diminishing step-size sequences, ensuring that after some time the change of their local policies does not grow significantly larger than the step size.
\\ 
The asymptotic properties of 3DPG are the following:
\begin{enumerate}
    \item All local critics converge to a set of DNN weights such that the local Bellman-error gradients are zero on average w.r.t. to limiting distributions over the global state-action space.
    \item All local actors converge to a set of DNN weights such that the local policy gradients 
    are also zero w.r.t. the same limiting distributions.
\end{enumerate}
The aforementioned limiting distributions are shaped the agents' accumulated local experiences. More specifically, the global data tuples available to agents through communication, which are then used in the training steps for critics and actors, form local limiting distributions such that all critics and actors converge jointly in expectation to stationary points with respect to these limiting distributions over the global state action space. 
In that regard, no agent can improve its performance locally w.r.t. its limiting distribution of the training process and the limiting policies of the other agents. Specifically, we show that 3DPG agents converge to a local Nash equilibrium of Markov games. Notably, this is achieved although every agent may have a different local reward structure, i.e. irrespective of whether they cooperate or compete.

\subsection{Main contributions and Paper Summary}
\begin{itemize}
    \item We present 3DPG, an online, fully distributed MAAC learning algorithm for networked MAS with continuous decision spaces.
    \item We present an asymptotic convergence analysis for 3DPG under practical conditions. These include novel communication assumptions that formalize how old information used by 3DPG agents are allowed to be.
    \item We show that 3DPG converges to a local Nash equilibrium of a Markov Game.
    \item Our convergence analysis, provides the first characterization of the policies found by Deep AC and Deep MAAC algorithms under assumptions that represent how these DeepRL algorithms are used in practice.
    \item Our analysis reveals the difference between using \emph{policies} of other agents (3DPG) compared to using \emph{actions} of other agents (MADDPG) in MAAC algorithms: in the presence of exploration MAAC algorithms that learn from actions take into account that other agents may act randomly even though they do not actually do so; 3DPG does not have this negative property.
\end{itemize}

To study the convergence of 3DPG, we use recent asymptotic analyses of Deep Q-Learning under practical and mild assumptions \cite{ramaswamy2021deep}. Since DDPG may be as a viewed special case of 3DPG involving a single agent, our analysis can be directly used to understand DDPG. 




\textbf{Paper organization.} In \Cref{sec:backgr-mark-games},  we define the Markov game setting considered in this work. In \Cref{sec: algo}, we describe the 3DPG Algorithm and the required assumptions to prove its convergence. \Cref{sec:analysis,sec: extensions} present our main convergence analysis. Afterwards, we analyze the difference between the 3DPG and MADDPG algorithm in \Cref{sec:comparison}, which we support with numerical experiments \cref{sec: numerical}. We close with conclusions and future work in \Cref{sec:discussions}. 

\section{Markov Games}
\label{sec:backgr-mark-games}

In multi-agent systems, the global state of the environment is typically the concatenation of the agents' local states. However, the global state is usually unobservable by any agent. The global state transitions to a new state after each agent has taken its local action. 
After the global state transition, the agents receive \emph{local feedback/reward signals}. The structure of the local reward signals depend on the nature of interaction between the agents: do they cooperate or compete?

In the following, we always use superscripts $i$ for the agent index in the multi-agent system, and we use subscripts $n$ as discrete time steps.
We assume that the multi-agent system under consideration can be modeled as a $D$-agent Markov game \cite{littman1994markov}, which is formally defined as the 4-tuple $\left(\cS, \cA, p, \{r^i \mid 1\le i \le D\}\right),$
where:
\begin{itemize}
	\item[$\cS$] $ = \prod_{i=1}^D \cS^i$ is the global state space, with $\cS^i \coloneqq \R^{k^i}$ the local state space of agent~$i$ with $k^i>0$.
	\item[$\cA$]  $= \prod_{i=1}^{D}\cA^i$ is the global action space, where an action $a = (a^1, \ldots, a^D) \in \cA$ denotes the joint action as a concatenation of local actions $a^i \in \cA^i \subseteq \R^{d^i}, d^i > 0$.
	\item[$p$] is the Markov transition kernel, i.e. $p(\cdot \mid s, a)$ is the distribution of the successor state of state $s$ after action $a$ is executed. 
	\item[$r^i$] is the local scalar reward function associated with agent $i$. Specifically, $r^i(s,a)$ is the local reward that agent $i$ observes when the system is in state $s$ and the global action $a$ is taken. 
\end{itemize}
In many cooperative Markov games, the local reward functions coincide, i.e., $r^i \equiv r$ for $1 \le i \le D$. Such models are called factored decentralized MDPs \cite{bernstein2002complexity}. 

Consider a $D$-agent Markov game as defined in above. The $D$ agents interact with the environment at discrete time steps $n \in \N$. At every time step $n$, agent~$i$ observes a local state $s^i_n \in \cS^i$, based upon which it must take a local (continuous) action $a^i_n \in \cA^i \subseteq \R^{d^i} $, for which it receives a reward $r_n^i$. 

Suppose that the local behavior of agent~$i$ is defined by a local DNN policy $\pi^{i}(s^i;\phi^i)$, parameterized by a vector $\phi^i$. Define the associated global policy as
\begin{equation}
\label{eq:global_policy}
    \pi \coloneqq (\pi^{1}, \ldots, \pi^{D}).
\end{equation} 
For each local reward function $r^i$, the return starting from time step $1$ is defined by $R^i \coloneqq \sum_{n=1}^{\infty} \gamma^{n-1} r^i(s_n,a_n)$ with discount factor $0 < \gamma < 1$. Given a global policy $\pi$, the associated action-value function $Q^i$  of agent~$i$ is given by $Q^i(s,a) \coloneqq \E_{\pi} \left[R^i \mid s_1 = s, a_1 = a \right]$. For each local reward function $r^i$, the associated optimal policy is characterized by the optimal action-value function $Q^i_*(s,a)$, which is defined as a solution to Bellman's equation \cite{bertsekas1996neuro}:
\begin{equation}
\label{eq:Bellmans_equation}
    Q^i_*(s,a) = \E_{s',r^i} \left[r^i(s,a) + \gamma \max_{a'\in \cA} Q^i_*(s',a') \mid s,a \right]
\end{equation}
The problem is to find local policy parametrizations $\phi^i_*$ for each local policy $\pi^{i}(s^i;\phi^i)$, such that for every agent~$i$:
\begin{equation}
    \label{eq:local_objective}
    \pi^i(s^i; \phi^i_*) \approx \underset{a^i \in \cA}{\text{argmax}} \ Q^i_*(s, a)\big\vert_{\pi^{j\not=i}} \quad \forall s\in \cS.
\end{equation}
In other words, all local policies should act optimally conditioned on the local policies obtained by all other agents. 

\section{Distributed Deep Deterministic Policy Gradient Algorithm (3DPG)}
\label{sec: algo}
In this section, we define the 3DPG algorithm.

\subsection{Multi-Agent Actor-Critic Gradients}
\label{subsec:MAACgradients}

Suppose each agent~$i$ uses a local DNN approximator $ Q^i(s,a ; \theta^i)$ for its local critic;  $\theta^i$  represents the associated vector of network weights. The local critic is trained using the Deep Q-Learning algorithm \cite{mnih2015human} to find $\theta^i_*$ such that $ Q^i(s,a ; \theta_*^i) \approx Q^i_*(s,a)$ for all state-action pairs $(s,a)$. As mentioned before, each local actor/policy $\pi^{i}(s^i; \phi^i)$ is parameterized by $\phi^i$. The goal is to train the local critics and local actors jointly such that \eqref{eq:local_objective} holds. This is challenging due to the potential conflicting rewards of the agents.

Consider that at some time step $n$, the local critic and actor parametrizations are $\theta_n^i$ and $\phi_n^i$. Further, suppose agent~$i$ gets access to a global data tuple 
\begin{equation}
\label{eq:global_data_tuple}
    t^i_m \coloneqq (s_m, a_m, r^i(s_m,a_m), s_{m+1})
\end{equation}
from some transition from time $m$ to $m+1$ with $m \ll n$. The availability of at least some global tuples must be ensured by coordinated communication between the agents. This will be discussed further in the following subsections. Now, if agent~$i$ had access to the parametrizations of the other agents, then it could calculate a local critic gradient
\begin{equation}
\begin{split}
\label{eq:central_critic_grad}
    &\nabla_{\theta^i} l^i(\theta^i_n, \phi_{n}, t^i_m) \coloneqq \nabla_{\theta^i} Q^i(s_m,a_m; \theta_n^i)\Big(r^i(s_m,a_m)  + \\& 
    \gamma Q^i(s_{m+1}, \pi(s_{m+1}; \phi_n); \theta_n^i)- Q^i(s_m,a_m; \theta^i_n)\Big).
\end{split}
\end{equation}
\Cref{eq:central_critic_grad} is a sample gradient of the local squared Bellman error of $Q^i$ for the observed global tuple $t_m$, which follows from the associated error in Bellmans equation \eqref{eq:Bellmans_equation}.

Now, there two possible ways to formulate a ``natural'' distributed version of the policy gradient in the DDPG algorithm \cite{lillicrap2015continuous}. 
The first one is the local policy gradient
\begin{equation}
\label{eq: lowe_grad_0}
    \begin{split}
        &\nabla_{\phi^i} g_{\text{MADDPG}}^i(\theta^i_n, \phi^i_{n}, s_m, a^{j\not=i}_m) \coloneqq  \nabla_{\phi^i} \pi^i(s_m^i; \phi^i_n) \\ &\qquad \nabla_{a^i} Q^i (s_m,a^1_m, \ldots, \pi^i(s^i_m; \phi^i_n), \ldots, a^D_m; \theta^i_n ),
    \end{split}
\end{equation}
which is used in the MADDPG algorithm of \cite{lowe2017multi}.
The second one is the local policy gradient
\begin{equation}
    \label{eq:central_actor_grad}
    \begin{split}
        &\nabla_{\phi^i} g_{\text{3DPG}}^i(\theta^i_n, \phi_{n}, s_m) \coloneqq \nabla_{\phi^i} \pi^i(s_m^i; \phi^i_n) \\ &\qquad \nabla_{a^i} Q^i (s_m,\pi(s_m; \phi_n)); \theta^i_n ),
    \end{split}
\end{equation}
which will be used in our 3DPG algorithm.
The MADDPG local policy gradient uses the actions $a_m^{j\not= i}$ from the other agents from the global tuple $t_m^i$, while the local policy gradients in 3DPG use the policies $\phi^{j\not= i}_n$ from the other agents.\emph{ In the following, we use $\nabla_{\phi^i} g^i$ for $\eqref{eq:central_actor_grad}$ to simplify the notation. }

\begin{remark}
The subtle difference between $\eqref{eq: lowe_grad_0}$ and \eqref{eq:central_actor_grad} will be analyzed extensively in \Cref{sec:comparison}. We discuss a heuristic argument based on our asymptotic convergence analysis of MADDPG and 3DPG that shows that a multi-agent actor-critic algorithm based on \eqref{eq:central_actor_grad} has a higher probability of obtaining a better policy faster compared to a multi-agent actor-critic algorithm based on \eqref{eq: lowe_grad_0}. This is because \eqref{eq:central_actor_grad} takes precisely into account how other agents would behave in certain sampled states. \Cref{eq: lowe_grad_0}, on the other hand, also considers the sampled actions that may arise from randomly explored actions. Our numerical experiment in \Cref{sec: numerical} supports this theoretical prediction. 
\end{remark}
\begin{remark}
The local policy gradient \eqref{eq:central_actor_grad} seems to have a more direct motivation from the deterministic policy gradient theorem (DPGT) \cite{silver2014deterministic} in comparison to \eqref{eq: lowe_grad_0}. The DDPG alogrithm was inspired by taking samples from the DPGT\footnote{Be aware that the sample gradients used in the DDPG algorithm are not true sample gradients from the deterministic policy gradient \cite{nota2019policy}}. Similarly, the gradient \eqref{eq:central_actor_grad} can be motivated by inserting the global product policy \eqref{eq:global_policy} into the DPGT. Thus \Cref{eq:central_actor_grad} is in essence the policy gradient from the DDPG algorithm \cite{lillicrap2015continuous}, where the policy is defined in product form \eqref{eq:global_policy}. 
\end{remark}

The idea behind 3DPG is to approximate \eqref{eq:central_critic_grad} and \eqref{eq:central_actor_grad} using old information from other agents to train $Q^i$ and $\pi^i$ locally. To implement this, the agents require: 
\begin{enumerate}
    \item Local access to global data tuples $t^i_m$,
    \item Local access to the global policy $\pi(s;\phi_m)$,
\end{enumerate}
for $m \ll n$ ``frequently" (the precise network assumptions are presented in \cref{sec:network_asm}). Recall that in MADDPG the above information are required for all agents at every time step. These are reasonable assumptions for simulated environments or under the paradigm of centralized training with decentralized execution, but not for online fully distributed learning. 

\subsection{Approximate Global Policy induced by local AoI}

We now decentralize the implementation of \eqref{eq:central_critic_grad} and \eqref{eq:central_actor_grad} using communication. Most notably, we use communicated but potentially aged local policies as an approximation for the true global policy $\phi_n$.

Suppose that the $D$-agent Markov game is networked, such that the agents can exchange data by communication. We suggest a communication paradigm where agents cooperatively forward local data to other agents, such that local policies and local data (states and actions) can flow via the network to all other agents. To guarantee this, the agents must use some forwarding protocol \cite{lim2001flooding} to forward old policies $\phi^i_n$ between the agents as well as 
coordinated communication protocol to ensure that at least some global tuples \eqref{eq:global_data_tuple} reach each agent ``frequently''. The coordinated communication protocol may be some broadcast protocol coupled with a central coordinator, or it could also be a distributed snapshot protocol \cite{chandy1985distributed}, which, however, would cost more communication resources.
For now, we suppose that the agents run suitable protocols of this kind. Specifically, protocols that guarantee that our network assumptions (A1), as to be defined in the next section, are satisfied. 

Let us now suppose that each agent runs a local algorithm to train its policy and thereby generates a sequence $\phi_n^i$ of associated policy parametrizations. Equipped with the ability to transfer data via the available network, the agents exchange the local parametrization $\phi_n^i$ as well as local tuples $t_n^i \coloneqq (s^i_n, a^i_n, s^i_{n+1})$ using the communication network
that possibly delays or looses data for extended periods of time.
Hence, agent~$i$ has only access to $\phi^j_{n-\tau_{ij}(n)}$ for every agent $j\not=i$ at every time step $n$. Here, $\phi^j_{n-\tau_{ij}(n)}$ denotes the latest available policy parametrization from agent $j$ at agent $i$ at time $n$ and we refer to $\tau_{ij}(n)$ as the associated \emph{Age of Information} (AoI) random variable as a consequence of the potentially uncertain and delaying communication.\footnote{For background on AoI we refer the reader to \cite{redder2022practical}, where the effect of AoI was consider in an offline distributed optimization setting.} For every agent~$i$, we can then define a global policy parametrization associated with the aged information at time $n$ by
\begin{equation}
\label{eq:AoI_policy}
\phi_{\tau^i(n)} \coloneqq ( \phi^1_{n-\tau_{i1}(n)}, \ldots,  \phi^D_{n-\tau_{D1}(n)} ). 
\end{equation}
This global policy will serve as an approximation to the true global policy $\phi_n$. 

\subsection{The 3DPG Algorithm}
\label{sec:algorithm}

As discussed in the previous section, we suppose that the $D$-agent multi-agent system uses an available network to exchange their local policy parametrizations $\phi_n^i$ and their local tuples $(s_n^i,a_n^i,s^i_{n+1})$. We can now state the 3DPG iteration.

Suppose that every agent~$i$ maintains a local replay memory $\cR^i_n$. At every time step $n$, the memory can contain up to $N$ old global transitions $t^i_m$. At time step $n$, agent~$i$ samples a random minibatch of $M < N$ transitions from its replay memory. Agent~$i$ then updates its actor and critic using step-size sequences $\alpha(n)$ and $\beta(n)$ as follows:
\begin{equation}
\begin{split}
\label{eq: AC_iteration}
\theta^i_{n+1} &= \theta^i_n + \alpha(n)\frac{1}{M} \sum_{m}  \nabla_{\theta^i} l^i(\theta^i_n, \phi_{\tau^i(n)}, t_m^i), \\
\phi^i_{n+1}&= \phi^i_n + \beta(n) \frac{1}{M} \sum_{m}\nabla_{\phi^i} g^i(\theta^i_m, \phi_{\tau^i(n)}, s_m)
\end{split}
\end{equation}
Notice that for a single sample $t_m^i$, the gradients used in \eqref{eq: AC_iteration} are the gradients \eqref{eq:central_critic_grad} and \eqref{eq:central_actor_grad} where the global policy $\phi_n$ has been replaced by the local approximation of the global policy $\phi_{\tau^i(n)}$ induced by the aged parametrization \eqref{eq:AoI_policy}. The resulting training architecture is presented in \Cref{fig:3DPG_flow}. Pseudocode for the algorithm is presented in \Cref{app:env&algo}. We will now present our assumptions to prove the convergence of \eqref{eq: AC_iteration}. We begin with the required network assumptions.

\subsection{Network Assumptions}
\label{sec:network_asm}
The communication network needs to ensure two things. First, it needs to ensure that every agent~$i$ receives the policy parametrizations $\phi^j_n$ for all $j\not= i$ ``sufficiently'' often. Second, it needs to ensure that the available samples in the replay memories $\cR_n^i$ are not too old. To capture the age of the samples in the replay memories, define another AoI random variable $\Delta^i(n)$ as the age of the oldest sample in the replay memory $\cR^i_n$ of agent~$i$ at time $n \ge 0$.
\begin{definition}
\label{def: stoch_dom}
A non-negative integer-valued random variable $X$ is said to be stochastically dominated by a random variable $\overline{X}$ if $\Pr{X > m } \le \Pr{\overline{X} > m}$ for all $m \ge 0$.
\end{definition}

\begin{itemize}
    \item[(A1)] (a) \textit{Policy communication assumptions:} \\
	There exists a non-negative integer-valued random variable $\overline{\tau}$ that stochastically dominates all $\tau_{ij}(n)$ for all $n \ge 0$ with $\Ew{\overline{\tau}^{q_1}} < \infty$ for some $q_1\ge 1$. 
	\item[(A1)] (b) \textit{Data communication assumptions:} \\
	There exists a non-negative integer-valued random variable $\overline{\Delta}$ that stochastically dominates all $\Delta^{i}(n)$ for all $n \ge 0$ with $\Ew{\overline{\Delta}^{q_2}} < \infty$ for some $q_2\ge 1$. 
\end{itemize}

(A1)(a) and (A2)(b) require that the tail distributions of the AoI variables $\tau_{ij}(n)$ and $\Delta^i(n)$ decay uniformly, such that at least a dominating random variable with finite mean exists. This ensures that the growth of each AoI variable cannot exceed any fraction of $n$ after some potentially large time step. We prove this in  \Cref{lem:finite_CDFsum} and show in \Cref{lem:AoI} that the use of the approximate global policies $\phi_{\tau^i(n)}$ 
in \eqref{eq: AC_iteration} does not cause gradient errors asymptotically. Finally, (A2)(b) is used in \Cref{lem:Markov}
to show that the agents experiences converge to a stationary distribution. Here (A2)(b) ensures that the agents receive enough global tuples asymptotically to ``track'' the Markov game state distribution.

\begin{remark}
(A1)(b) does not specify when exactly the global samples become available to each agent. Further, the received data tuples do not have to be from the same time steps $m$ for every agent.
\end{remark}

\subsection{Algorithm and Markov Game Assumptions}
\label{sec:assumptions}
In addition to the network assumptions, we require:
\begin{itemize}
	\item[(A2)] (a) The critic step size sequence $\alpha(n)$ is positive, monotonically decreasing and satisfies: \\ $$\sum_{n\ge 0} \alpha(n) = \infty \quad \text{and} \quad \sum_{n\ge 0} \alpha^2(n) < \infty.$$ \\
	(b) The actor step size sequence $\beta(n)$ satisfies:
	$$\beta(n) \in \cO(n^{-\frac{1}{q}}) \quad \text{and} \quad
	 \lim\limits_{n \to \infty} \frac{b(n)}{a(n)} = 1, $$
	 for $q\in [1,2)$ with $q \le \min(q_1,q_2)$ for $q_1, q_2$ from (A1).
	\item[(A3)] (a)  $\sup_{n\ge 0} \norm{\theta_n} < \infty$ a.s. and  $\sup_{n\ge 0} \norm{\phi_n} < \infty$ a.s.\\
	(b) $\sup_{n\ge 0} \norm{s_n} < \infty$ a.s. and the action space $\cA$ is compact.
	\item[(A4)] The state transition kernel $p(\cdot \mid s,a)$ is continuous.
	\item[(A5)] The actor policies $\pi^i$ and the critics $Q^i$ are fully connected feedforward neural networks with twice continuously differentiable activation functions such as Gaussian Error Linear Units (GELUs).
	\item[(A6)] The reward functions $r^i: \cS \times \cA \rightarrow \R$ are continuous.
\end{itemize}

Assumption (A3)(a) is the strongest assumption as it requires almost sure stability of the algorithm. We devote the next subsection to its discussion. The compactness of the action space in (A3)(b) will usually be satisfied in many applications, e.g. in robotics.

In (A2)(b) we require that the critic and actor step-size sequences are chosen such the $\frac{\beta(n)}{\alpha(n)} \rightarrow 1$. This is not a traditional assumption for actor-critc algorithms \cite{borkar1997actor}. We will present a proof based on a single-timescale analysis of \eqref{eq: AC_iteration} w.r.t. the timescale of the critic iterations. In practice, we want the critic to converge faster so we would initially choose $\alpha(n)$ larger than $\beta(n)$. (A2)(b) requires that afterwards the iterations asymptotically take steps of the same size. The more complex analysis using a two-timescale step-size schedule will be presented in an upcoming paper.

In (A5), we require twice continuously differentiability of the activations used by the policy and actor networks. GELUs are well-known examples that satisfy this property \cite{hendrycks2016gaussian}. Additionally, GELUs are one of the well-known neural network activation functions with similar high performance across different tasks compared to other well-known activations like ELUs or LeakyReLUs \cite{ramachandran2017searching}.

\subsection{Discussion of Assumption (A3)(a) and Related Work}

Ideally, when one is dealing with a specific algorithm, it should be guaranteed or proven -- rather than assumed up-front -- that the algorithm iterations are stable. Assuming stability is, nonetheless, a typical first step towards understanding the convergence behavior of optimization algorithms. Especially in deep RL, stability of algorithms like Deep Q-Learning or DDPG are not well understood. Most notably, there is a significant gap between the assumptions made in theory compared to assumptions verifiable in practice. 
Lets review some results on MAAC learning.

In \cite{zhang2018fully} the authors use linear function approximation and assume that the MA learning problem can be described by finite state ergodic Markov process. They further assume assume the existence of projection operator with knowledge of a compact set that includes a local minima of the objective. 
\cite{kumar2019sample} provides very interesting rate of convergence results for AC methods. However, they assume that samples $(s_n,a_n,r_n,s_{n+1})$ are drawn from a known stationary distribution of the state Markov process. We instead show that our AC iterates converge such that the experience of the agents give rise to stationary distributions of the state Markov process. In addition, knowledge of the bias of the policy gradient and the bias of the critic estimates is required in \cite{kumar2019sample}, while the critic should again be a linear combination of features. That work also assumes that the policy gradient is Lipschitz continuous, which would require (A3)(a), since most DNNs are only locally Lipschitz.

The assumptions made in the above works will be very hard to verify for most data-driven applications in practice. Even worse, we fear that guaranteeing stability for practical data-driven RL problems may always require assumptions that are not easily verifiable in practice. However, a practitioner may not even be highly interested in stability. Usually, practitioners will design their DNN parameterizations and their hyperparameter configurations using their experience, such that they roughly observe stable behavior. 
Afterwards, practitioners would like to know what limit they can expect from their algorithm. This is where our work comes into play.

In contrast to the assumptions made in the literature, our assumptions, except (A3)(a), are very week, easily verifiable in practice and represent well how users apply DQN, DDPG and its variants in practice. For this setting, our work answers to where one can expect the 3DPG iterations \eqref{eq: AC_iteration} to converge asymptotically. Specifically, our analysis gives a comprehensive characterization of the found limit using a limiting distributions of the state-action process. These limiting distributions are shown to be stationary distributions of state Markov process and are shaped by the experience of the agents. We now present our convergence result.

\section{Main Results}

\begin{theorem}
\label{thm:main}
Under (A1)-(A6), the 3DPG iterations \eqref{eq: AC_iteration} converge to $\theta^i_\infty$ and $\phi^i_\infty$, such that 
\begin{equation}
\begin{split}
    &\nabla_{\theta^i} \left( \int_{\cS \times \cA} l^i (\theta^i_\infty, \phi_\infty, s, a) \mu_\infty^i(ds,da) \right) = 0, \\
    &\nabla_{\phi^i} \left(  \int_{\cS} g^i(\theta^i_\infty, \phi_\infty, s) \mu^{i}_\infty(ds, \cA) \right) = 0, 
\end{split}
\end{equation}
where $\mu^i_\infty$ is a limiting distribution of a continuous time measure process (defined in \Cref{sec: extensions}) that captures the experience of agent~$i$ sampled from its local experience replay $\cR^i_n$ during training. Further, all $\mu^i_\infty$ are stationary distributions of the state Markov process:
\begin{equation}
    \mu^i_\infty(dy \times \cA) = \int_\cS p(dy \mid x , \pi(x;\phi_\infty)) \mu^i_\infty(dx\times \cA).
\end{equation}
\end{theorem}

\Cref{thm:main} shows that the critic iterations of 3DPG converge to stationary points of the average local squared Bellmann errors. Further, the actor iterations converge to stationary points of the average local deterministic policy gradients.
For both limits the averaging is w.r.t. to the stationary distributions of the state Markov process that capture the experienced samples of the agents.

Since stochastic gradient descent schemes tend to avoid unstable equlibria
\cite{mertikopoulos2020almost,vlaski2021second,ge2015escaping}, we can expect that the aforementioned stationary points are local minima with high probability. This can be made more precise using an avoidance of traps analysis \cite[Ch. 4]{borkar2009stochastic}.
This shows that 3DPG converges to local solution of the objective \eqref{eq:local_objective} with high probability. In other words, given their local reward structure the agents converge to an equilibrium where they have locally no desire to change their policies given their local experience and the final policies of the other agents. More precisely, it follows that the agents converge to a local Nash equilibrium w.r.t. to their locally approximated action-value functions.

Abbreviate the final local policies as $\pi^i_\infty(s^i) \coloneqq \pi^i(s^i;\phi^i_\infty)$. For any open set $\cU$ with $0 \in \cU$ in the parameter space of $\pi^i_\infty(s^i)$, define
\begin{equation}
    \Pi^i_\infty(\cU) \coloneqq \{ \pi^i(\ \cdot \ ; \phi^i) : \phi^i \in \phi^i_\infty + \cU \}.
\end{equation}
In other words, $\Pi^i_\infty(\cU)$ is the set of policies in the $\cU$ neighborhood of $\pi^i_\infty(s)$.
\begin{corollary}
\label{cor:nash}
Suppose the stationary points from \Cref{thm:main} are local minima, then there are open sets $\cU^i$ with $0 \in \cU^i,$ such that
\begin{align}
    &\int_\cS Q^i(s,\pi^1_\infty(s), \ldots, \pi^i_\infty(s),  \ldots, \pi^D_\infty(s); \theta^i_\infty) \mu^i_\infty(s,\cA)\\
    &\ge \int_\cS Q^i(s,\pi^1_\infty(s), \ldots, \pi^i(s),  \ldots, \pi^D_\infty(s); \theta^i_\infty) \mu^i_\infty(s,\cA) \nonumber
\end{align}
for all $\pi^i \in \Pi^i_\infty(\cU^i)$.
\end{corollary}
\Cref{cor:nash} shows that the local policies converge to a local approximate Nash equilibrium w.r.t. the experience gathered by each agent locally. The experience is again represented by the limiting distributions $\mu^i_\infty$. Specifically, the local policies converge to a local Nash equilibrium w.r.t. the local expected action-value functions
$\int_\cS Q^i(s, \cdot \ ; \theta^i_\infty) \mu^i_\infty(s,\cA)$. In game theoretic terms, these are the payoff (or utility) functions w.r.t. which the agents converge to a local Nash equilibrium. 

\begin{remark}
\Cref{cor:nash} holds when the local policies $\pi^i$ are linear functions of pretrained non-linear features. This is common in the literature as e.g. used in the discussed references \cite{zhang2018fully} and \cite{kumar2019sample}. The significance of \Cref{cor:nash} is that the agents converge to a local approximate Nash equilibrium without assuming that the samples used in training are from a known stationary distribution of the state Markov process. We instead show that the experience of the agents give rise to stationary distributions of the state Markov process. This is important, as deep RL algorithms are typically employed in complex environments with multiple stationary distributions
\end{remark}


\section{Preliminaries and Age of Information Analysis}
\label{sec:analysis}

In the following two sections, we prove \Cref{thm:main}.
The proof builds on the analysis of single agent deep Q-learning presented in \cite{ramaswamy2021deep}. To simplify the presentation, we will assume from now on that the state space $\cS$ is compact. All results can be generalized to $d$-dimensional real spaces under the almost sure boundedness condition in (A3)(b), for which we refer to techniques presented in \cite[Section IV.A.2]{ramaswamy2021deep}.

At its core, we will now present a convergence proof for the DDPG algorithm \cite{lillicrap2015continuous}, using a single timescale analysis. We begin with preliminary reductions and the analysis of the AoI processes $\tau_{ij}(n)$. 

\subsection{Reduction to mini batches of size 1}
First, we make a simplifying reduction. We consider that the agents have ready access to the global tuples $(s_n,a_n,r^i(s_n,a_n),s_{n+1})$ during runtime and that merely the local policies $\phi^i_n$ are communicated via the communication network. Further, we merely consider that the agents use the global tuple from time $n$ to update its critic and actor network. We therefore simplify iteration \eqref{eq: AC_iteration} to:
\begin{equation}
\begin{split}
\label{eq: AC_iteration_analysis}
    \theta^i_{n+1} &= \theta^i_n + \alpha(n) \nabla_{\theta^i} l(\theta^i_n, \phi_{\tau^i(n)}, t_n^i),\\
    \phi^i_{n+1} &= \phi_n + \beta(n) \nabla_{\phi^i} g(\theta^i_n, \phi_{\tau^i(n)}, s_n).
\end{split}
\end{equation}
In \Cref{sec: extensions}, we will extend our analysis to the setting presented in \Cref{sec: algo}. 

\subsection{Reduction to zero AoI} 

As the second step, we define the gradient errors that occur since we use the aged global policies $\phi_{\tau^i(n)}$ instead of the true global policy:
\begin{equation}
    \begin{split}
        \label{eq:grad_erros}
        e_n^{\theta^i} &\coloneqq \nabla_{\theta^i} l(\theta^i_n, \phi_{n}, t_n) - \nabla_{\theta^i} l(\theta^i_n, \phi_{\tau^i(n)}, t_n)\\
        e_n^{\phi^i} &\coloneqq \nabla_{\phi^i} g(\theta^i_n, \phi_{n}, s_n) - \nabla_{\phi^i} g(\theta^i_n, \phi_{\tau^i(n)}, s_n)
    \end{split}
\end{equation}
Hence, \eqref{eq: AC_iteration_analysis} can be written as 
\begin{equation}
\begin{split}
\label{eq: AC_iteration_analysis_2}
    \theta^i_{n+1} &= \theta^i_n + \alpha(n) \left( \nabla_{\theta^i} l(\theta^i_n, \phi_{n}, t_n) + e_n^{\theta^i}\right),\\
    \phi^i_{n+1} &= \phi_n + \beta(n) \left( \nabla_{\phi^i} g(\theta^i_n, \phi_{n}, s_n) + e_n^{\phi^i}\right).
\end{split}
\end{equation}

\subsection{Reduction to marginalized critic gradient}

As the third step, we rewrite the critic iterations in \eqref{eq: AC_iteration_analysis_2} further by integrating over the the successor state $s_{n+1}$ in $t_n$ given state $s_n$. The resulting new loss gradient is
$\nabla_{\theta^i} \hat{l}^i(\theta^i, \phi, s, a) \coloneqq$
\begin{equation}
    \begin{split}
    \label{eq:new_Bellman}
         &\Big( r^i(s,a)  + \gamma \int Q^i(s', \pi(s'; \phi); \theta^i) p(ds' \mid s,\phi) \\ &- Q^i(s, a; \theta^i) \Big) \nabla_{\theta^i} Q^i(s,a; \theta^i),
    \end{split}
\end{equation}
With a slight abuse of notation, we use $p(ds \mid s_n, \phi_n)$ instead of $p(ds \mid s_n, \phi_n(s_n))$ to highlight the dependency of the action $a_n$ on the policy $\phi_n$ and potential additional random noise for exploration. 

Define the induced error from using $\nabla_{\theta^i} \hat{l}^i$ instead of $\nabla_{\theta^i} l^i$ 
as $\psi^i_n \coloneqq \nabla_{\theta^i} \hat{l}^i - \nabla_{\theta^i} l^i$, we can then rewrite \eqref{eq: AC_iteration_analysis_2} as
\begin{equation}
\begin{split}
\label{eq: AC_iteration_analysis_3}
    \theta^i_{n+1} &= \theta^i_n + \alpha(n) \left( \nabla_{\theta^i} \hat{l}^i(\theta^i_n, \phi_{n}, s_n, a_n) + \psi^i_n + e_n^{\theta^i}\right),\\
    \phi^i_{n+1} &= \phi_n + \beta(n) \left( \nabla_{\phi^i} g(\theta^i_n, \phi_{n}, s_n) + e_n^{\phi^i}\right).
\end{split}
\end{equation}

\subsection{$e_n^{\theta^i}$, $e_n^{\phi^i}$ and $\psi^i_n$ vanish asymptotically}
In summary, we have rewritten \eqref{eq: AC_iteration_analysis} using:
\begin{enumerate}
    \item The errors $e_n^{\theta^i}$ and $e_n^{\phi^i}$ induced by not considering the AoI random variables $\tau_{ij}(n)$,
    \item The errors $\psi^i_n$ induced by marginalizing out the successor states $s_{n+1}$. 
\end{enumerate}
We will now show that these errors vanish asymptotically. For this, we first present properties of the loss gradients $\nabla_{\phi^i} g^i$ , $\nabla_{\theta^i} l^i$ and $\nabla_{\theta^i} \hat{l}^i$.
\begin{lemma}
\label{lem: local_lip}
	(i) $\nabla_{\theta^i} l^i(\theta^i_n, \phi_{n}, t_n)$ and $\nabla_{\theta^i} \hat{l}^i(\theta^i_n, \phi_{n}, s_n, a_n)$ are continuous and locally Lipschitz continuous in the $\theta^i$ and $\phi$-coordinate. \\
	(ii) $\nabla_{\phi^i} g(\theta^i_n, \phi_{n}, s_n,a_n )$ is locally Lipschitz continuous in every coordinate.
\end{lemma}
\begin{proof}
    See Appendix \ref{app:proofs}.
\end{proof}
We can now show that the errors $e_n^{\theta^i}$ and $e_n^{\phi^i}$ due to AoI vanish asymptotically. For this we need a technical lemma.
\begin{lemma}
	\label{lem:finite_CDFsum}
	Under (A1)(a) it follows that for every $\varepsilon \in (0,1)$ and all agent pairs $(i,j)$, 
	\begin{equation}
		\sum_{n=0}^{\infty} \Pr{\tau_{ij}(n) > \varepsilon n^{\frac{1}{q_1}}} \le \frac{1}{\varepsilon^{q_1}} \Ew{\overline{\tau}^{q_1}} < \infty.
	\end{equation}
	with $\overline{\tau}$ from (A1)(a).
\end{lemma}
\begin{proof}
    See Appendix \ref{app:proofs}.
\end{proof}
The lemma shows that the AoI processes $\tau_{ij}(n)$ do not exceed any fraction of $n^\frac{1}{q_1}$ asymptotically. More precisely, it now follows from the Borel-Cantelli Lemma that $\Pr{\tau_{ij}(n) > \varepsilon n^{\frac{1}{q_1}} \text{ i.o.}} = 0.$
Hence, there is sample path dependent $N(\varepsilon) \in \N$, such that
	\begin{equation}
		\label{eq: ana3}
		\tau_{ij}(n) \le \varepsilon n^{\frac{1}{q_1}} \qquad \forall \, n\ge N(\varepsilon).
\end{equation}
The following Lemma shows that the gradient errors vanish asymptotically.

\begin{lemma}
\label{lem:AoI}
   $\lim\limits_{n \to \infty} \norm{e_n^{\theta^i}} = 0$ and $\lim\limits_{n \to \infty} \norm{e_n^{\phi^i}} = 0$.
\end{lemma}
\begin{proof}
    From Lemma \ref{lem: local_lip} we have that $\nabla \hat{l}^i$ is  locally Lipschitz. It follows from (A3)(a) that $\nabla \hat{l}^i$ is Lipschitz continuous with constant $L$ when restricted to a sample path dependent compact set. 
	Using the triangular inequality, the established Lipschitz continuity of $\nabla \hat{l}^i$ and (A3)(a), it follows that
	\begin{equation} 
		\label{eq: ana_1}
		\begin{split}
		    \norm{e^{\theta^i}_n} &\le  L \sum_{j \not= i} \sum \limits_{m = n - \tau_{ij}(n)}^{n-1} \norm{\phi^j_{m+1} - \phi^j_m} \\ &\le C \sum_{j \not= i} \sum \limits_{m = n - \tau_{ij}(n)}^{n-1} \beta(m),
		\end{split}
	\end{equation}
	for a sample path dependent constant $C > 0$.
	We will now show that
	\begin{equation}
		\label{eq: ana_claim}
		\lim\limits_{n \to 0 }  \left( \sum \limits_{m = n - \tau_{ij}(n)}^{n-1} \beta(m) \right) = 0,
	\end{equation}
	which thus implies that $\lim\limits_{n \to \infty} \norm{e^{\theta^i}_n} = 0$.
	
	By (A2)(b), we assume that $\beta(n) \in \cO(n^{-\frac{1}{q}})$ with $q\le q_1$
	Hence, there are constants $c>0$ and $N \in \N$, such that 
	\begin{equation}
		\label{eq: ana2}
		\beta(n) \le c n^{-\frac{1}{q_1}} \text{ for all } n\ge N.
	\end{equation}
	Fix $\varepsilon\in (0,1)$. \Cref{eq: ana2,eq: ana3} show that
	\begin{equation}
		\sum \limits_{m = n - \tau_{ij}(n)}^{n-1} \beta(m) \le c \sum \limits_{m = n - \varepsilon n^{\frac{1}{q_1}} }^{n-1}  m^{-\frac{1}{q_1}}
	\end{equation}
	for all $n$ with $n \ge N(\varepsilon)$ and $n - \varepsilon n^{\frac{1}{q_1}} \ge N$.
	Using the monotonicity of $n^{-\frac{1}{q_1}}$, it follows that
	\begin{equation}
	\begin{split}
	    \sum \limits_{m = n - \varepsilon n^{\frac{1}{q_1}} }^{n-1}  m^{-\frac{1}{q_1}} &\le \varepsilon n^{\frac{1}{q_1}} (n - \varepsilon n^{\frac{1}{q_1}} )^{-\frac{1}{q_1}} 
	\end{split}
	\end{equation} 
	Taking the $\limsup$ on both sides above yields that
	\begin{equation}
		\limsup\limits_{n\rightarrow \infty} \left( \sum \limits_{m = n - \tau_{ij}(n)}^{n-1} \beta(m) \right) \le
		\frac{c\varepsilon}{1-\varepsilon}.
	\end{equation}
	The statement follows since the choice of $\varepsilon$ was arbitrary. The proof for $e_n^{\phi^i}$ follows analogously.
\end{proof}

The next lemma shows that the accumulated 
errors due to the marginalization of the successor states in \eqref{eq: AC_iteration_analysis_3} is convergent almost surely. It therefore follows that $\psi^i_n$ vanishes.
\begin{lemma}
\label{lem:Martingale}
    $\Psi^i_n \coloneqq \sum_{m=0}^{n-1} \alpha(m) \psi^i_n$ is a zero-mean square integrable martingale. Hence,  $\Psi^i_n$ converges almost surely.
\end{lemma}
\begin{proof}
    See Appendix \ref{app:proofs}.
\end{proof}

It now follows from Lemma~\ref{lem:AoI} and Lemma~\ref{lem:Martingale} that we can study the convergence of \eqref{eq: AC_iteration_analysis_3} without the additional error terms $e_n^{\theta^i}$, $e_n^{\phi^i}$ and $\psi^i_n$. This is because the error terms will contribute additional asymptotically negligible errors in the following proof of Lemma~\ref{lem: StochApproximation_lemma}. With a slide abuse of notation we now redefine the critic loss gradients $\nabla_{\theta^i} l^i$ as the marginalized critic loss gradient $\nabla_{\theta^i} \hat{l}^i$.
 
\section{Convergence Analysis}
\label{sec:convergence-proof}

To analyze the asymptotic behavior of \eqref{eq: AC_iteration_analysis_3}, we follow the ODE approach from SA \cite{borkar2009stochastic}, i.e. we
construct a continuous-time trajectory with the same limiting behavior as \eqref{eq: AC_iteration_analysis_3}. First, we divide the time axis using $\alpha(n)$ as follows:
\begin{equation}
    t_0 \coloneqq 0 \text{ and } t_n \coloneqq \sum_{m=0}^{n-1}\alpha(m) \text{ for } n \ge 1.
\end{equation}
Now define 
\begin{equation}
\label{eq: linear^interpol}
    \overline{\theta}^i(t_n) \coloneqq \theta^i_n, n \ge 0 \text{ and } \overline{\phi}^i(t_n) \coloneqq \phi^i_{n}, n \ge 0.
\end{equation}
Let $\R^{p_\theta^i}$ and $\R^{p_\phi^i}$ be the parameter spaces of the $\theta^i_n$'s and $\phi^i_{n}$'s, respectively.
Then define $\overline{\theta}^i \in \C([0,\infty), \R^{p_\theta^i})$ and $\overline{\phi}^i \in \C([0,\infty), \R^{p_\phi^i})$ by linear interpolation of all $\overline{\theta}^i(t_n)$ and $\overline{\phi}^i(t_n)$, respectively.

To analyze the training process, we formulate a measure process that captures the encountered state-action pairs when using the global policy $\pi(s_n; \phi_n)$. Therefore, define
\begin{equation}
\label{eq: measure}
    \mu(t) = \delta_{(s_n,a_n)}, t\in [t_n, t_{n+1}]
\end{equation}
where $\delta_{(x,a)}$ denotes the Dirac measure. This defines a process of probability measures on $\cS \times \cA$. For every probability measure $\nu$ on $\cS \times \cA$, define
\begin{equation}
\begin{split}
\label{eq: average_gradients}
    \tilde{\nabla}  l^i (\theta^i, \phi, \nu) &\coloneqq  \int \nabla_{\theta^i}l^i (\theta^i, \phi, s, a) \nu(ds,da),\\
    \tilde{\nabla}  g^i (\theta^i, \phi, \nu) &\coloneqq  \int \nabla_{\phi^i}g^i (\theta^i, \phi, s) \nu(ds, \cA). 
\end{split}
\end{equation}
Note that in $\nabla_{\phi^i}g^i$ we used $\nu(ds, \cA)$, since the actor update in \eqref{eq: AC_iteration_analysis_3} is only state-dependent.
It follows from Lemma \ref{lem: local_lip} that all $\tilde{\nabla} l^i$ and $\tilde{\nabla} g^i$ are continuous in all coordinates and locally Lipschitz in both the $\theta^i$- and $\phi$-coordinate.

We can now define the associated continuous time trajectories in $\C([0,\infty), \R^{p_\theta^i})$ and $\C([0,\infty), \R^{p_\phi^i})$ that capture the training process starting from time $t_n$ for $n\ge0$:
\begin{equation}
\begin{split}
    \label{eq: int_trajectory}
    \theta^i_n(t) &\coloneqq \overline{\theta}^i(t_n) + \int_0^t \tilde{\nabla} l^i (\theta_n^i(x), \phi_n(x), \mu_n(x)) dx, \\
    \phi^i_n(t) &\coloneqq \overline{\phi}^i(t_n) + \int_0^t \tilde{\nabla} g^i (\theta_n^i(x), \phi_n(x), \mu_n(x)) dx.
\end{split}
\end{equation}
where $\mu_n(t) \coloneqq \mu(t_n +t) $. The combination of the continuous-time trajectories in \eqref{eq: int_trajectory} results in the aforementioned single trajectory with the same limiting behavior as \eqref{eq: AC_iteration_analysis_3}.  

Per definition, the trajectories define solutions to the following families of non-autonomous ordinary differential equations (ODEs):
\begin{equation}
\begin{split}
    \label{eq: ODEs}
    \{ \Dot{\theta}^i_n(t) &= \tilde{\nabla} l^i (\theta_n^i(t), \phi_n(t), \mu_n(t)) \}_{n\ge0}, \\
    \{ \Dot{\phi}^i_n(t) &= \tilde{\nabla} g^i (\theta_n^i(t), \phi_n(t), \mu_n(t)) \}_{n\ge0}.
\end{split}
\end{equation}
By construction, we obtain that the limiting behavior of \eqref{eq: AC_iteration_analysis_3}
is captured by the limits of the sequences $\{ \overline{\theta}^i([t_n,\infty)) \}_{n\ge0} $ and $\{ \overline{\phi}^i([t_n,\infty)) \}_{n\ge0} $ defined by \eqref{eq: linear^interpol}. 
Further, the sequences defined in \eqref{eq: int_trajectory} can be analyzed as solutions to the ODEs in \eqref{eq: ODEs}. If \eqref{eq: linear^interpol} and \eqref{eq: int_trajectory} behave asymptotically identical, then the limiting behavior of \eqref{eq: AC_iteration_analysis_3} is thus captured by the solutions to the ODEs in \eqref{eq: ODEs}. This is formalized by the following important technical Lemma~\ref{lem: StochApproximation_lemma}. This lemma is the key component to enable a single-timescale analysis of DDPG style actor-critic algorithms along the line of argument presented in \cite{ramaswamy2021deep} for Deep Q-Learning.
To prove  Lemma~\ref{lem: StochApproximation_lemma}, we use that the step size sequences are related by $\frac{\beta(n)}{\alpha(n)} \rightarrow 1$ from (A2)(b). This is essential since we just constructed the continuous trajectories w.r.t. the timescale induced by $\alpha(n)$. The assumption in essence requires that the critic and actor updates asymptotically run on the same time scale. 
\begin{lemma}
\label{lem: StochApproximation_lemma}
For every $T>0$, we have 
\begin{align}
    \lim\limits_{n\rightarrow \infty} \sup\limits_{t\in [0,T]} \norm{\overline{\theta}^i(t_n +t) - \theta^i_n(t)} &= 0, \\
    \lim\limits_{n\rightarrow \infty} \sup\limits_{t\in [0,T]} \norm{\overline{\phi}^i(t_n +t) - \phi^i_n(t)} &= 0 \label{eq: stochapp_phi}.
\end{align}
\end{lemma}
\begin{proof}
Fix $T >0$. We define $[t]$ for $t\ge0$ as $[t] \coloneqq t_{\sup\{n \mid t_n \le t\}} $. Fix $t\in [0,T]$, then $[t_n +t] = t_{n+k}$ for some $k\ge0$. 
Recall, that $\overline{\phi}^i(t)$ is defined by linear interpolation w.r.t. $\alpha(n)$, see \eqref{eq: linear^interpol}. Hence, $\overline{\phi}^i(t_n+t)- \overline{\phi}^i(t_{n+k})$ is equal to
\begin{equation}
     \frac{t_n+t-t_{n+k}}{\alpha(n+k)} \left(\overline{\phi}^i(t_{n+k+1}) - \overline{\phi}^i(t_{n+k}) \right).
\end{equation}
The stability of the algorithm and the compactness of the state-action space, $i.e.$ (A3), show that $\nabla_{\phi^i}g^i$ is bounded and hence $\norm{\overline{\phi}^i(t_{n+k+1}) - \overline{\phi}^i(t_{n+k}) } \in \cO (\beta(n+k))$.
It follows that
\begin{equation}
\label{eq: eq_1}
    \sup\limits_{t\in[0,T]} \norm{\overline{\phi}^i(t_n +t) - \overline{\phi}^i( [t_n + t])} \in \cO(\alpha(n)), 
\end{equation}
since $\alpha(n)$ is monotonic and  $\frac{\beta(n)}{\alpha(n)} \rightarrow 1$. Similarly, we can show that
\begin{equation}
\label{eq: eq_2}
    \sup\limits_{t\in[0,T]} \norm{\phi^i_n(t) - \phi^i_n([t_n+t] -t_n)} \in \cO(\alpha(n))
\end{equation}
To show \eqref{eq: stochapp_phi}, we now need to show that
\begin{equation}
\label{eq: eq_3}
   \sup\limits_{t\in[0,T]}  \norm{\overline{\phi}^i([t_n +t]) - \phi^i_n([t_n +t] -t_n)} \rightarrow 0.
\end{equation}
From \eqref{eq: int_trajectory} it follows that
\begin{equation}
    \begin{split}
    \label{eq: first^ineq}
        &\norm{\overline{\phi}^i(t_{n+k}) - \phi^i_n(t_{n+k} -t_n)} 
     \le \lVert \overline{\phi}^i(t_{n+k}) -  \overline{\phi}^i(t_{n}) \\ &- \int_0^{t_{n+k} -t_{n}} \tilde{\nabla} g^i (\theta_n^i(x), \phi_n(x), \mu_n(x)) dx \lVert.
    \end{split}
\end{equation}
Using a telescoping series, $\overline{\phi}^i(t_{n+k}) - \phi^i_n(t_{n+k} -t_n)$ equals
\begin{equation}
\begin{split}
\label{eq: first_term}
    & \sum_{m=n}^{n+k-1} \beta(m) \nabla_{\phi^i} g^i (\theta^i_m, \phi_m, s_m) \\
    &= \sum_{m=n}^{n+k-1} \int_{t_m}^{t_{m+1}} \frac{\beta(m)}{\alpha(m)}\tilde{\nabla} g^i (\overline{\theta}^i([x]), \overline{\phi}([x]) , \mu_n(x - t_n)) dx
\end{split}
\end{equation}
The last step follows from $\alpha(m) = t_{m+1} - t_m$ and using that $\phi^i_m = \overline{\phi}^i(t_m) = \overline{\phi}^i([t])$ for all $t \in [t_m, t_{m+1})$.
Now rewrite the second term in \eqref{eq: first^ineq}:
\begin{equation}
    \begin{split}
    \label{eq: second_term}
        &\int_0^{t_{n+k} -t_{n}} \tilde{\nabla} g^i (\theta_n^i(x), \phi_n(x), \mu_n(x)) dx =\\ &\sum_{m=n}^{n+k-1}  \int_{t_m}^{t_{m+1}} \tilde{\nabla} g^i (\theta_n^i(x-t_n), \phi_n(x-t_n), \mu_n(x-t_n)) dx
    \end{split}
\end{equation}
We now evaluate the difference of the terms under the integrals in \eqref{eq: first_term} and \eqref{eq: second_term}. 
\begin{equation}
    \begin{split}
    \label{eq: difference}
        & \lVert \frac{\beta(m)}{\alpha(m)} \tilde{\nabla} g^i (\overline{\theta}^i([x]), \overline{\phi}([x]) , \mu_n(x - t_n)) \\ & \quad - \tilde{\nabla} g^i (\theta_n^i(x-t_n), \phi_n(x-t_n), \mu_n(x-t_n))  \lVert 
        \\& \le  C \lvert \frac{\beta(m)}{\alpha(m)} - 1 \lvert
        + \lVert \tilde{\nabla} g^i (\overline{\theta}^i([x]), \overline{\phi}([x]) , \mu_n(x - t_n)) \\ & \quad - \tilde{\nabla} g^i (\theta_n^i(x-t_n), \phi_n(x-t_n), \mu_n(x-t_n))  \rVert 
        \\& \le  C \lvert \frac{\beta(m)}{\alpha(m)} - 1 \lvert + L \Big( 
        \norm{\overline{\theta}^i([x]) - \overline{\theta} ^i_n([x]-t_n)} \\
        & +
        \norm{\overline{\phi}([x]) - \phi_n([x]-t_n)}  + \norm{\theta^i_n(x-t_n) -  \theta^i_n([x]-t_n)} \\ &\qquad + \norm{\phi_n(x-t_n) -  \phi_n([x]-t_n)} \Big)
    \end{split}
\end{equation}
for some sample path dependent constant $C< \infty$ using the stability from (A3). The last step adds zeros and uses the Lippschitz continuity of $\tilde{\nabla} g^i$.
The combination of \eqref{eq: first^ineq}, \eqref{eq: first_term}, \eqref{eq: second_term} and \eqref{eq: difference} thus gives:
\begin{equation}
    \begin{split}
    \label{eq: final}
        &\norm{\overline{\phi}^i(t_{n+k}) - \phi^i_n(t_{n+k} -t_n)}\\  &\le \sum_{m=n}^{n+k-1} \alpha(m) \cO\left(\Big\lvert \frac{\beta(m)}{\alpha(m)} - 1 \Big\lvert \right)
        + L \sum_{m=n}^{n+k-1} \cO\left(a(m)^2\right)
        \\
        &+ L \sum_{m=n}^{n+k-1} a(m) \Big( 
        \norm{\overline{\theta}^i(t_m) - \overline{\theta} ^i_n(t_m-t_n)}   \\ &\quad + \sum^i \norm{\overline{\phi}^i(t_m) - \overline{\phi}^i_n(t_m-t_n)}\Big)
    \end{split}
\end{equation}
The first term in the above expression converges to zero as $n\rightarrow \infty$, since $\sum_{m=n}^{n+k-1} \alpha(m) \le T$ by construction and since $\lim\limits_{n \to \infty} \frac{b(n)}{a(n)} = 1$ from (A2)(b).
The second term converges to zero, since $\alpha(n)$ is square summable (A2)(a). 

Inequality \eqref{eq: final} can now be derived analogously for $\norm{\overline{\theta}^i(t_{n+k}) - \theta^i_n(t_{n+k} -t_n)}$.
We can now sum up all L.H.S. and R.H.S. for all $i$ in \eqref{eq: final}, and for all $\theta^i$:
\begin{equation}
\begin{split}
\label{eq:sum_gronwall}
    &x_n \coloneqq \sum_{i=1}^D \norm{\overline{\phi}^i(t_{n+k}) - \phi^i_n(t_{n+k} -t_n)} \\ + &\sum_{i=1}^D \norm{\overline{\theta}^i(t_{n+k}) - \theta^i_n(t_{n+k} -t_n)} \\
    \le & \ o(1) + 2 L \sum_{m=n}^{n+k-1} a(m) \sum_i 
        \norm{\overline{\theta}^i(t_m) - \overline{\theta}^i_n(t_m-t_n)}   \\ &\quad + 2LD \sum_{m=n}^{n+k-1} a(m)  \sum_i \norm{\overline{\phi}^i(t_m) - \overline{\phi}^i_n(t_m-t_n)},
\end{split}
\end{equation}
where $D$ is the number of agents.
We now apply the discrete version of Gronwall inequality \cite{borkar2009stochastic} to $x_n$. It follows that $x_n \le o(1) e^{2LD \sum_{m=n}^{n+k-1} a(m)} $. By construction $\sum_{m=n}^{n+k-1} a(m) \le T$ for all $n\ge0$, thus
$x_n \to 0$, which proves the lemma.

\end{proof}

\Cref{lem: StochApproximation_lemma} shows that we can analyze the limits of \eqref{eq: AC_iteration_analysis_3} as the limits of the continuous-time trajectories defined in \eqref{eq: linear^interpol} in conjunction with the measure process \eqref{eq: measure}. By construction, the trajectories $\theta^i_n(t)$ and $\phi^i_n(t)$ are equicontinuous. Moreover, they are point-wise bounded from (A3)(a).
It now follows from the Arzela-Ascoli theorem, \cite{billingsley2013convergence}, that the families of trajectories
\begin{equation}
    \{\theta^i_n([0,\infty)) \}_{n=0}^\infty, \qquad \{\phi^i_n([0,\infty)) \}_{n=0}^\infty
\end{equation}
are sequentially compact in $\C([0,\infty), \R^{p_\theta^i})$ and $C([0,\infty), \R^{p_\phi^i})$, respectively. Further, it can be shown that the space of measurable functions from $[0,\infty)$ to the space of probability measures on $\cS \times \cA$ is compact metrizable \cite{borkar2006stochastic}. It now follows that the product space of all trajectories $\theta^i_n(t)$ and $\phi^i_n(t)$ together with the aforementioned space of measurable functions is sequentially compact. Hence, there is a common subsequence such that all considered sequences converge simultaneously, i.e. we obtain (with a slight abuse of notation) that $\theta^i_n \rightarrow \theta^i_\infty$ in $C([0,\infty), \R^{p_\theta^i})$, $\phi^i_n \rightarrow \phi^i_\infty$ in $C([0,\infty), \R^{p_\phi^i})$ and $\mu_n \rightarrow \mu_\infty$ in the space of measurable functions.  Analogously to \cite[Lemma 4]{ramaswamy2021deep}, we can show that $\mu_n(t)$ also converges in distribution to $\mu_\infty(t)$ in $\P(\cS\times \cA)$, $t$ almost everywhere. 

The following lemma now shows that the limits $\theta^i_\infty$, $\phi^i_\infty$ and $\mu_\infty$ are solutions to the limits of the families of non-autonomous ordinary differential equations \eqref{eq: ODEs}. 
\begin{lemma}
\label{lem: ode_solution}
a)  $\theta^i_\infty$ is  a solution to  $\Dot{\theta}^i(t) = \tilde{\nabla} l^i (\theta^i(t), \phi^1_\infty(t), \ldots, \phi^D_\infty(t), \mu_\infty(t)) $ \\b)  $\phi^i_\infty$ is  a solution to $\Dot{\phi}^i(t) =\tilde{\nabla} g^i (\theta^i_\infty(t), \phi_\infty^1(t), \ldots, \phi^i(t),\ldots, \phi_\infty^D(t),  \mu_\infty(t)) $
\end{lemma}
\begin{proof}
    See Appendix \ref{app:proofs}.
\end{proof}

We can now study the limit of 3DPG as a solution to the aforementioned non-autonomous ODE's. Specifically, append the ODEs to form a new ODE in the appended parameter space. The rest of the analysis follows the line of argument in \cite[Thm. 1]{ramaswamy2021deep}, so we only state the main conclusion. Let $(\theta^1, \ldots, \theta^D, \phi^1, \ldots, \phi^D)$ be a solution to the appended ODE, then the solution converges to an equilibrium of the appended ODE, i.e. $\tilde{\nabla} l^i (\theta^i, \phi, \overline{\mu}^i_\infty) = 0$ and  $\tilde{\nabla} g^i(\theta^i, \phi, \overline{\mu}^i_\infty) = 0$ for all $i$,
where $\lim\limits_{t\to \infty} \mu^i_\infty(t) \overset{d}{=} \overline{\mu}^i_\infty$.
\Cref{lem: StochApproximation_lemma} and \Cref{lem: ode_solution} show that the joint of the sequences $\overline{\theta}^i(t_{n(k)})$ and $\overline{\theta}^i(t_{n(k)})$ are solutions to the appended ODE for $\{n(k)\}_{k\ge0} \subset \{n\}_{n\ge0}$. The last two statements thus show that the limits of $\overline{\theta}^i(t_{n(k)})$ and $\overline{\theta}^i(t_{n(k)})$, let us call them $\overline{\theta}^i_\infty$ and $\overline{\phi}^i_\infty$, are equilibrium points of the ODE's in \Cref{lem: ode_solution}. These limits determine the long-term behavior
of 3DPG. 

Finally, the first part of Theorem \ref{thm:main} now follows (for the particular case of experience replay with size 1 and global information access without communication) using (A3) to swap the order of differentiation and integration in $\tilde{\nabla} l^i$ and $\tilde{\nabla}g^i$. 
It left to show \Cref{thm:main} for 3DPG with experience replays and only local information access, and to show that the limiting distributions $\overline{\mu}^i_\infty(s, \cA)$ are stationary w.r.t. the state Markov process. Both are the subject of the next section.

\section{Extension to Experience Replays}
\label{sec: extensions}

Experience replay buffers play an important role in stabilization of RL algorithms \cite{mnih2015human}. The fundamental idea is to learn from past experiences to reduce the bias of an RL algorithm towards the interactions of an agent with its environment. For 3DPG this means that at time $n$ an agent does not use the transition $t_n^i$ to calculate the loss gradients, but it uses a random minibatch of past transitions $t_m^i$ from old time steps $m\le n$. As a consequence, the training algorithm is not overtly biased agents interaction with the environment, reducing learning variance, thereby improving stability.

In the previous section, we analyzed 3DPG for centralized training where the global transitions $t_n^i$ are locally available for every agent~$i$. Additionally, we only used an experience replay of size one. To accommodate the use of experience replays in \Cref{sec:analysis}, the probability measure $\mu(t)$ needs to be redefined. In \eqref{eq: AC_iteration}, each agent~$i$ samples $M<N$ global tuples independently from its (local) random experience replay $\cR^i_n$ at every iteration $n$. The sampling processes of the agents will in general be different. Further, we will experience $\cR^i_n \not= \cR^j_n$, since the global tuples are communicated by the agents in a potentially delaying manner. 

We now define a new measure process $\mu^i(t)$ for each agent. For $t \in [t_n, t_{n+1})$, define 
\begin{equation}
\label{eq:new_measures}
    \mu^i(t) \coloneqq \frac{1}{M} \sum_{j=1}^M \delta_{(s_{m(n,j,i)}, a_{m(n,j,i)})}.
\end{equation}
Hence, $\mu^i(t)$ is the probability measure on $\cS \times \cA$ that places a mass of $\nicefrac{1}{M} $ on each pair $(s_{m(n,j,i)}, a_{m(n,j,i)})$ for $1\le j \le M$, where each $m(n,j,i)$ denotes one of the time indices sampled by agent $i$ at time $n$ from its memory $\cR^i_n$. 
Notice that in the presence of communication and AoI, each experience replay is a random sequence of sets. 
If we use the redefined measures in \eqref{eq: average_gradients} we get for every $t = t_n$ that
\begin{equation}
\begin{split}
    &\tilde{\nabla}  l^i (\theta^i(t), \phi(t), \mu^i(t)) \\ &= \frac{1}{M} \sum_{j=1}^M \nabla_{\theta^i} l^i(\theta^i_n, \phi_n, s_{m(n,j,i)}, a_{m(n,j,i)}),\\
    &\tilde{\nabla}  g^i (\theta^i(t), \phi(t),  {\mu^{i}}(t)) = \frac{1}{M} \sum_{j=1}^M \nabla_{\phi^i} g^i (\theta^i_n, \phi_n, s_{m(n,j,i)} ).
\end{split}
\end{equation}
The analysis presented in \Cref{sec:analysis} is also true for then new measure processes, where now every agent has its own measure process. This shows the first part of \Cref{thm:main}. It is left to characterize the properties of the limiting measure processes $\mu^i_\infty$, which are the limits of the convergent subsequence extracted from $\mu^i_n(t) \coloneqq \mu^i(t_n+t)$. 

\begin{lemma}
    \label{lem:Markov}
    For all $t\in [0,\infty)$ and for all agents~$i$
    \begin{equation}
    \mu^i_\infty(t,dy \times \cA) = \int_\cS p(dy \mid x , \phi_\infty(t)) \mu^i_\infty(t,dx\times \cA).
\end{equation}
In other words, the limiting marginals constitute stationary distributions over the state Markov process.
\end{lemma}

\begin{proof}
    Without loss of generality, assume a batch-size $M=1$. The cases $M>1$ will only require additional bookkeeping. Recall that the samples used in the 3DPG iterations \eqref{eq: AC_iteration} are potentially old and from random time-steps, such that (A1)(b) holds.\footnote{Notably, (A1)(b) in conjunction with the Borel-Cantelli lemma guarantee that infinitely many global transitions reach each agent.} Fix some agent~$i$. In the following, we will drop the agent index $i$. Since $M=1$, the agent uses a global transition $t_{k_n}$ with random time index $k_n$ for its 3DPG training step at time $n$.

    Pick $f \in C_b(\cS)$, the convergence determining class for distributions on $\cS$. We analyze
    \begin{align}
        \label{eq:terms_markov_analysis}
        &\int_{t_n}^{t_{n+1}} \left[ f\left(s\right) - \int_{\cS} f(y) p\left(dy \mid s, \phi_{n}\right) \right] \mu(z, ds,\cA) dz \nonumber \\
        &= \alpha(n)\left[ f\left(s_{k_n} \right) - \int_{\cS} f(y) p\left(dy \mid s_{k_n}, \phi_{n}\right) \right]
    \end{align}
    as $n \to \infty$. 
    The error terms \eqref{eq:terms_markov_analysis} consider the deviation between states sampled from $\mu(t)$ and the associated expected transition under the \emph{policy that uses the sample during training}. This is the perspective of an experimenter that observes the Markov game and the 3DPG algorithm during runtime. 

    Now accumulate the aforementioned deviations for all time steps where a sample $t_n$ would be used during training and evaluate the deviation under the policy at time $n$. This information is of course not required for the algorithm and solely a quantity for the analysis. 
    In summary, consider
    \begin{equation}
        \label{eq:experimenter}
        \gamma(n) \left[f\left(s_{n+1}\right) - \int_{\cS} f(y) p\left(dy \mid s_{n}, \phi_{n}\right) \right]
    \end{equation}
    at time $n$, where 
    \begin{equation}
        \gamma(n) \coloneqq \sum_{i \in \{i \ge 0 \mid k_i = n\}} \alpha(i).
    \end{equation}
    for every $n \ge 0$ with $\sum_{i\in \emptyset} = 0$. Recall that $\alpha(n) \in \cO(n^{-\frac{1}{q}})$ with $q\in [1,2)$, then
    \begin{equation}
        \sum_{k=0}^n\gamma(k) \le \cO\left(\sum_{k=0}^{2n} \alpha(k) \right) \le \cO \left(\int_1^{2n} x^{-\frac{1}{q}} dx \right)
    \end{equation}
    since at time $n$ a sample from the replay memory can be at most $n$ time steps old. It then follows that $\frac{1}{n} \sum_{k=0}^n\gamma(k) \in \cO\left(n^{-\frac{1}{q}}\right) $, which in turn implies that $\gamma(k) \in \cO\left(n^{-\frac{1}{q}}\right) $ and thus that 
     $\gamma(m)$ is square summable.
    
    To analyze $\eqref{eq:experimenter}$, we now consider the sequence
    \begin{equation}
        \label{eq: stat_lemma_eq1b}
        \xi_n \coloneqq \sum_{m = 0}^{n-1} \gamma(m) \left[f\left(s_{m+1}\right) - \int_{\cS} f(y) p\left(dy \mid s_{m}, \phi_{m}\right) \right]
    \end{equation}
    and the filtration 
    \begin{equation}
        \cF_{n-1} \coloneqq \sigma \langle s_m,a_m,\phi_m, \gamma(m-1) \mid m \le n-1 \rangle.
    \end{equation}
    Then $\xi_n$ is a Martingale w.r.t $\cF_{n-1}$. 
    Since $f$ is bounded and $\gamma(m)$ is square summable, the quadratic variation process associated with the Martingale $\xi_n$ is convergent. It then follows from the Martingale Convergence Theorem \cite{durrett2019probability} that $\xi_n$ converges almost surely.
   
    Recall that we denote by $\Delta(n)$ the age of the oldest sample in the replay memory. \Cref{lem:finite_CDFsum} now also holds for $\Delta(n)$ using (A1)(b), i.e. fix $\varepsilon \in (0,1)$ then there is a sample path dependent constant $N(\varepsilon) \in \N$, such that $\Delta(n) \le \varepsilon n^{\frac{1}{q_2}}$ for all $n\ge N(\varepsilon)$ with $q_2$ from (A1)(b). In other words,
    \begin{equation}
    \label{eq:interval_kn}
        k_n \in [n- \varepsilon n^{\frac{1}{q_2}}], \text{ for all } n\ge N(\varepsilon).
    \end{equation}
    
    Since $\xi_n$ converges, it follows that for every $t>0$
    \begin{equation}
        \label{eq: stat_lemma_eq2}
        \sum_{m=n- \varepsilon n^\frac{1}{q_2}}^{\delta(n,t)} \gamma(m) \left[f\left(s_{m+1}\right) - \int_{\cS} f(y) p\left(dy \mid s_{m}, \phi_{m}\right) \right]
    \end{equation}
    converges to zero a.s., where $$\delta(n,t) \coloneqq \min \{m\ge n \mid t_m \ge t_n + t\}.$$
    Next, spread out and rearrange the aggregated samples in \eqref{eq: stat_lemma_eq2}. Specifically, use  \eqref{eq:interval_kn} and
    separate the samples as the following three terms, whose sum converges to zero a.s.: 
    \begin{equation}
        \begin{split}
            &\sum_{m=n}^{\delta(n,t)} \alpha(m) \left[f\left(s_{k_m+1}\right) - \int_{\cS} f(y) p\left(dy \mid s_{k_m}, \phi_{k_m}\right) \right] \\
        &+ \cO\left( \sum_{m=n-\varepsilon n^\frac{1}{q_2}}^n \alpha(m) \right) + \cO\left( \sum_{m=\delta(n,t)}^{\delta(n,t)+\varepsilon \delta(n,t)^\frac{1}{q_2}} \alpha(m) \right)
        \end{split}
    \end{equation}
    Notably, the rearrangement is mathematically valid since \eqref{eq: stat_lemma_eq2} is a finite sum for each $n$.
    Since $q_2 \ge  q$, it follows that the second and third term converge to zero. The proofs use the same line of argument as to show \eqref{eq: ana_claim} in \Cref{lem:AoI}.
    We can therefore conclude that
    \begin{equation}
    \label{eq: stat_lemma_eq4}
        \sum_{m=n}^{\delta(n,t)} \alpha(m) \left[f\left(s_{k_m+1}\right) - \int_{\cS} f(y) p\left(dy \mid s_{k_m}, \phi_{k_m}\right) \right]
    \end{equation}
    converges to zero a.s. 
    
    Equation \eqref{eq: stat_lemma_eq4} now also holds, if we replace $\phi_{k_m}$ by $\phi_m$ since the resulting error terms when taking the difference between \eqref{eq: stat_lemma_eq4} and the version with $\phi_m$ converges to zero. To see this, note that $\sum_{m=n}^{\delta(n,t)-1} \alpha(m) \in \cO(t)$ by construction. Further, all individual error terms in the aforementioned difference converge to zero using weak convergence by continuity of $(s,\phi) \mapsto p(\cdot \mid s, \phi)$ and since $\norm{\phi_{k_m}- \phi_m} \to 0 \text{ a.s.}$
    Finally, $\alpha(n)$ is eventually decreasing, hence
    $\sum_{m=n}^{\delta(n,t)}\left[ \alpha(m) - \alpha(m+1)\right] f(s_{k_m + 1}) \to 0 \text{ a.s.}$. In summary, we have thus shown that
    \begin{equation}
    \label{eq: stat_lemma_eq4b}
        \sum_{m=n}^{\delta(n,t)} \alpha(m) \left[f\left(s_{k_m}\right) - \int_{\cS} f(y) p\left(dy \mid s_{k_m}, \phi_{m}\right) \right]
    \end{equation}
    converges to zero a.s. 
    With \eqref{eq:terms_markov_analysis} and \eqref{eq: stat_lemma_eq4b} it then follows that
    \begin{equation}
        \label{eq: stat_lemma_eq5}
        \begin{split}
            \int_{t_n}^{t_n+t} \int_{\cS} \Big[f\left(s\right) - h(z,s)\Big] \mu(z, ds, \cA) dz \to 0 \text{ a.s.}
        \end{split}
    \end{equation}
    where $h(z,s) \coloneqq \int_{\cS} f(y) p\left(dy \mid s, \overline{\phi}(z) \right).$ The lemma now follows from \eqref{eq: stat_lemma_eq5}. We refer to \cite[Lemma 6]{ramaswamy2021deep} for details on this final step.
\end{proof}

\section{Cooperative Training of MAS based on Old Actions vs. Old Policies (MADDPG vs. 3DPG)}
\label{sec:comparison} 

In this section, we discuss the difference between 3DPG and MADDPG \emph{for the centralized training scenario} with global information access.

\subsection{The MADDPG policy gradient iteration}
An MADDPG agent~$i$ updates its policy using the gradient
\begin{equation}
\label{eq:MADDPG}
    \frac{1}{M} \sum_{m} \nabla_{\phi^i} g_{\text{MADDPG}}^i(\theta^i_n, \phi^i_{n}, s_m, a^{j\not=i}_m)
\end{equation}
for $1\le m \le M$ sampled transitions, with  $\nabla_{\phi^i} g_{\text{MADDPG}}^i$ as defined in \eqref{eq: lowe_grad_0}. Please again observe the difference compared to the second iteration  in our 3DPG algorithm \eqref{eq: AC_iteration}. 
In MADDPG old actions $a^j_m$ for all $j\not=i$ are used from the samples of the experience replay. In contrast the true current policy $a_j = \pi_{\phi^j_n}(s_m^j)$ is used in the 3DPG policy gradient iteration (assuming the centralized setting). MADDPG still uses the policies of other agents in the critic iteration. Hence, availability of the policies from other agents is anyway required.

Lets try to understand the subtle difference between 3DPG and MADDPG intuitively.
From the perspective of some agent~$i$ the MADDPG policy gradient iteration
appends the product action space of other agents $\prod_{j\not= i}\cA_j$ to the global state space $\cS$. 
For illustration, suppose agent~$i$ samples transitions $\{t_m\}_{m=1}^M$, and lets suppose all local states are equal, i.e. $s^i_m = s^i$. Then \eqref{eq:MADDPG} gives $\nabla_{\phi^i} \pi_{i}(s^i;\phi^i_n) $ times
\begin{equation}
\label{eq:MADDPG heuristic}
\left(  \frac{1}{M} \sum_{m} \nabla_{a^i} Q^i(s^i,a_m^1, \ldots, \pi_{i}(s^i; \phi^i), \ldots, a^D_m)  \right)
\end{equation} 
as the sample policy gradient. The expression averages over sampled actions of the other agents. These actions will sometimes include random actions due to exploration.\footnote{In most DeepRL algorithms the probability to select a random action is decayed to a small value over time. However, it is usually kept positive to also allow some asymptotic exploration.} This indicates that agents would learn polices that also act well for random behavior of other agents. This seems to be undesirable for cooperative learning. We will now make the above heuristic precise using \Cref{thm:main} and its variant for MADDPG. 

\subsection{The MADDPG limit vs. the 3DPG limit}
\label{sec:MADDPG_withExp}

First, we discuss the analog of \Cref{thm:main} for MADDPG
Specifically, consider \eqref{eq: AC_iteration} with $\nabla_{\phi^i} g_{\text{MADDPG}}^i$ as defined in \eqref{eq: lowe_grad_0} instead of 
$\nabla_{\phi^i} g_{\text{3DPG}}^i$ as defined in \eqref{eq:central_actor_grad}.

To analyze MADDPG with experience replays, we use the same measure processes \eqref{eq:new_measures} as in the \Cref{sec: extensions}. However, we need to redefine the average policy gradient in \eqref{eq: average_gradients} to
\begin{equation}
    \begin{split}
        \tilde{\nabla} g_{\text{MADDPG}}^i&(\theta^i, \phi^i, \nu) \coloneqq \int \nabla_{\phi^i} g_{\text{MADDPG}}^i (\theta^i, \phi^i_n, s, a^{j\not=i})  \\ & \nu(ds, da^1, \ldots, \cA^i, \ldots, da^D).
    \end{split}
\end{equation}
The analysis from \Cref{sec:analysis,sec: extensions}  can now be emulated for this gradient with the measure processes $\mu^i(t)$. Due to the new average policy gradient the conclusion of the convergence theorem are now fundamentally different. MADDPG converges to limits $\theta_{\text{MA}}^i$ and $\phi_{\text{MA}}^i$, such that
\begin{equation}
\label{eq:MADDPG_limit}
     \tilde{\nabla}g_{\text{MA}}^i(\theta_{\text{MA}}^i, \phi_{\text{MA}}^i, \mu_{\text{MA}}^i) = 0,
\end{equation}
Here, $\mu_{\text{MA}}^i$ are the local limiting distribution of the sampled experience at agent~$i$ under MADDPG. Recall that the 3DPG limit satisfies:
\begin{equation}
    \label{eq:3DPG_limit}
        \tilde{\nabla} g_{\text{3DPG}}^i(\overline{\theta}_\infty^i, \overline{\phi}_\infty, \overline{\mu}^i_\infty) = 0
\end{equation}

In \eqref{eq:MADDPG_limit}, the behavior of the other agents is solely present in the limiting measures $\mu_{\text{MA}}^i$. When exploration is stopped after some time, then the asymptotic properties of MADDPG and 3DPG are the same. However, when the exploration probability is not decayed to zero asymptotically, the the limiting measures $\mu^i_{MA}$ are also shaped by random actions. This formalizes the heuristic from the previous subsection:
the presence of exploration can deteriorate the policies found by MADDPG agents. The agents would adapt their policies to random actions that are not representative of the other agents. This is clearly an undesirable property. 3DPG fares better in this regard as it does not have this negative property! \textbf{3DPG allows a high exploration probability asymptotically without a negative effect on the found policies.} Next, we discuss the moving target problem. 

\begin{remark}
In practice, the aforementioned negative property of MADDPG due random actions is exacerbated due to the fact that training is stopped in finite time, possibly in a premature manner. During training the other agents may have initially behaved in a certain way, which was then well represented in the replay memory an agent. During later stages of learning, the other agents may ``quickly" converge to a different policy. Since the agent uses outdated samples to calculate local gradients, the policy evolution is significantly biased towards the old behavior of the other agents. This is particularly undesirable in cooperative problems. Again, 3DPG circumvents such scenarios by the using the latest available agent policies. MADDPG can overcome these issues by stopping exploration after some time and by using decaying learning rates.
\end{remark}

\subsection{The moving target problem in MARL}
\label{sec:moving_target}

Multi-agent RL algorithms are effected by non-stationarity due to the change of other agents behavior from the perspective of one agent (the so called moving target problem). Here, 3DPG is no exception.
The authors of \cite{lowe2017multi}, discuss that using the local policy gradients $\nabla_{\phi^i} g_{\text{MADDPG}}^i$ based on old actions of other agents removes the non-stationarity from the perspective of each agent. 
We believe that this not wholly true. MADDPG smooths out the non-stationarity, as the behavior of other agents is observed via the sampled actions from the experience replays. But this does not remove the non-stationarity. 
Intuitively, it takes longer for the behavior of an agent to manifest in the replay memory (in form of samples) compared to directly using the behavior of an agent using its policy as in 3DPG. 

Since 3DPG does use the policies of other agents, we expect that 3DPG has more variance due to the changing behavior of other agents. However, since the 3DPG agents use more accurate information from the other agents' policies and are not affected by exploration as described in the previous section, we expect a faster convergence rate compared to MADDPG. Both predictions, are validated in our experiments in the following \Cref{sec: numerical}. Finally, notice that the moving target problem has no impact on \Cref{cor:nash}. The 3DPG agents converge to a local Nash equilibrium.

\section{Numerical Experiments}
\label{sec: numerical}

For our numerical experiments, we first compare 3DPG and MADDPG in a centralized training setting with global information access. Afterwards, we evaluate 3DPG with communication, where local states, actions and policies have to be communicated.

\subsection{Experiment 1}

For the comparison of 3DPG and MADDPG, we consider a multi-agent coordination problem. We add an additional coordination layer to a version of the simple spread multi particle environment (SSMPE) presented in \cite{lowe2017multi}. In the SSMPE, particles (the agents) move in the plane to cover a number of landmarks and the landmarks are episodically reset to new positions. The agents get a global reward $$\exp{(-\text{average closest distance to the landmarks})} \in (0,1).$$ In this particular scenario, we did not observe any significant difference between 3DPG and MADDPG. This is because in SSMPE, the agents' actions do not require any form of coordination and the global reward at every time-step is only a function of the global state. 

We now add another coordination layer to the environment. We now consider two agents that move around in the plane with the objective to minimize their average distance to three landmarks as before. However, the agents only get rewarded if they move/watch in the same direction. Specifically, the previous global reward gets weighted by
\begin{equation}
\label{eq:angle}
    \exp{(-\text{angle between the orientation of the agents})}.
\end{equation}
The new reward structure therefore requires that the agents coordinate their decisions. More details and algorithm hyperparameters used for training are presented in Appendix \ref{app:env&algo}.

We trained both 3DPG and MADDPG with global data access and global access to all agents' policies, i.e. for the centralized training setting without communication and AoI. We trained the agents for 10 seeds over 1500 epochs, where each epoch had a horizon of 25 steps. \Cref{fig: comp} shows the the resulting average reward per epoch. 
\begin{figure}
\centering
\includegraphics[width=.47 \textwidth]{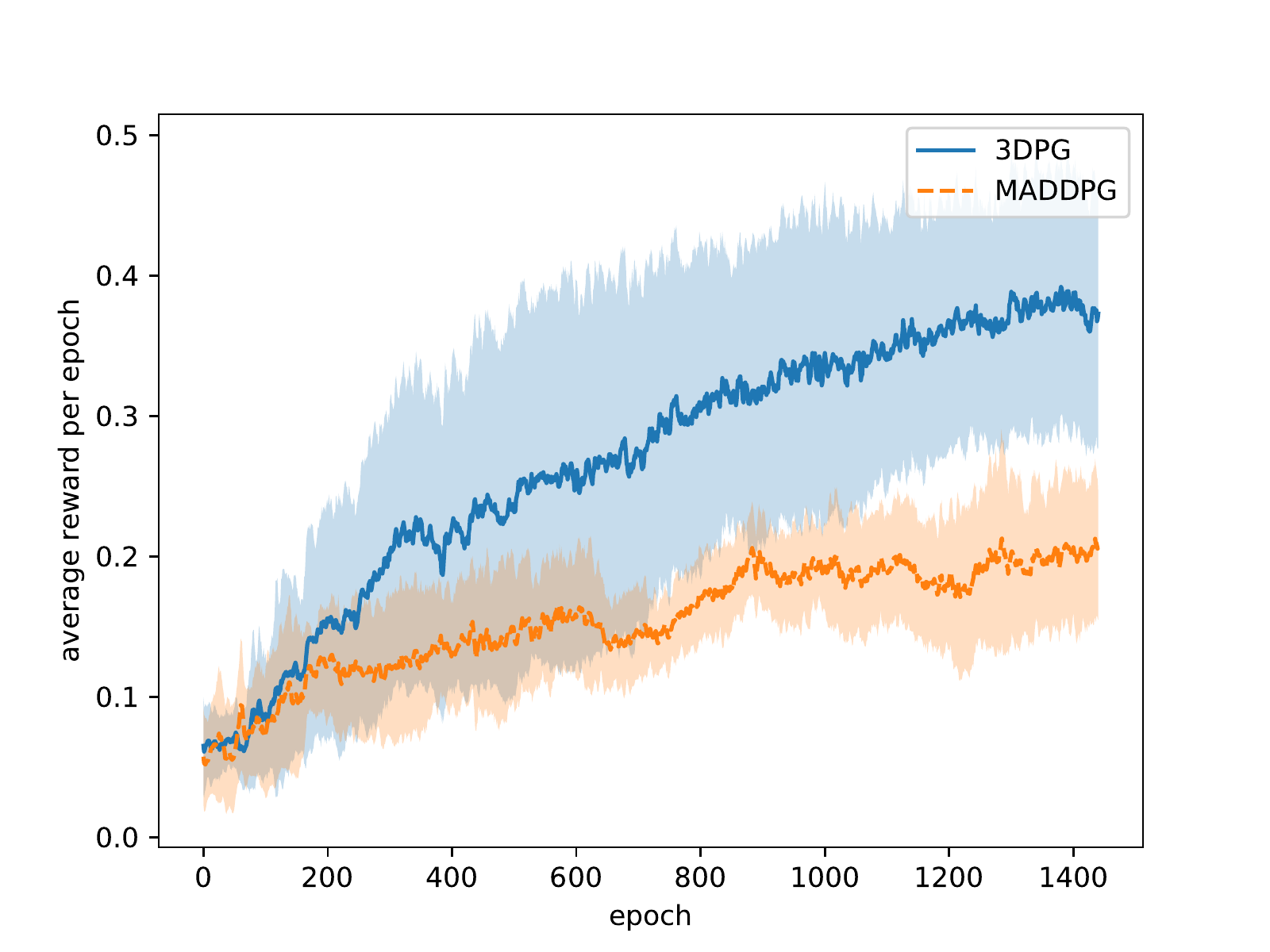}
\caption{Comparison of 3DPG and MADDPG with centralized training. }
\label{fig: comp}
\end{figure}
Our simulations support precisely our theoretical predictions from the previous section: \textbf{For problems that require coordinated actions, 3DPG obtains better policies faster than MADDPG at the cost of higher training variance.}

\subsection{Experiment 2}

In our second experiment, we show that 3DPG with communication is
robust to large AoI and that 3DPG may even benefit from using older policies of other agents similar to how target networks improve training in single agent RL \cite{lillicrap2015continuous}. 

We again consider the two agent, three landmark problem from Experiment 1. In addition, we consider that each of the two agents uses an independent communication channel for communication. Specifically, each agent has a fixed communication budget of 15000 bits/slot to communicate with the other agent whenever their channel access is success. We emulate lossy communication by varying the channel access probability $\lambda \in \{e^{-1}, e^{-2}, e^{-3}, e^{-4} \} \approx \{0.3679, 0.1353, 0.0498, 0.0183 \} $. This simple communication model satisfies our network assumptions (A1). For 3DPG, we use the same hyper-parameter configurations as in Experiment 1 (Appendix \ref{app:env&algo}). For this setting, the agents therefore require at least 3 successive successful communication events to exchange a parameter vector $\phi_n^i$, while at least 33 local tuples $(s^i_n,a^i_n, s^i_{n+1})$ could be exchanged during the same time. We let each agent cycle between communicating a policy update followed by communicating 33 local tuples. In that sense we give ``equal weight'' to policy and data communication.

In \Cref{fig:comparison_aoi}, we show the average reward per epoch.
We see that 3DPG with even $\lambda=e^{-4}$  is able to learn decent policies compared to centralized 3DPG, albeit at a slower convergence rate. Notably, the $\lambda=e^{-4}$ run achieves this with AoIs frequently  over 500 time steps (20 epochs) as shown in \Cref{fig:aoi}. In addition, 3DPG with $\lambda=e^{-4}$ has only access to $\nicefrac{1}{3}$ of the global data tuples that are used by 3DPG with centralized training. This shows that 3DPG is highly robust to AoI and low data availability. Finally, an interesting observation is that the 3DPG runs with $\lambda=e^{-1},e^{-2}$ or $e^{-3}$ consistently performed similar or even better than 3DPG centralized. This indicates that 3DPG may even benefit from using older policies of other agents. 

\begin{figure}
\centering
\subfigure[Average reward per epoch of 3DPG with variance over seeds.]{%
\label{fig:comparison_aoi}%
\includegraphics[width=.47 \textwidth]{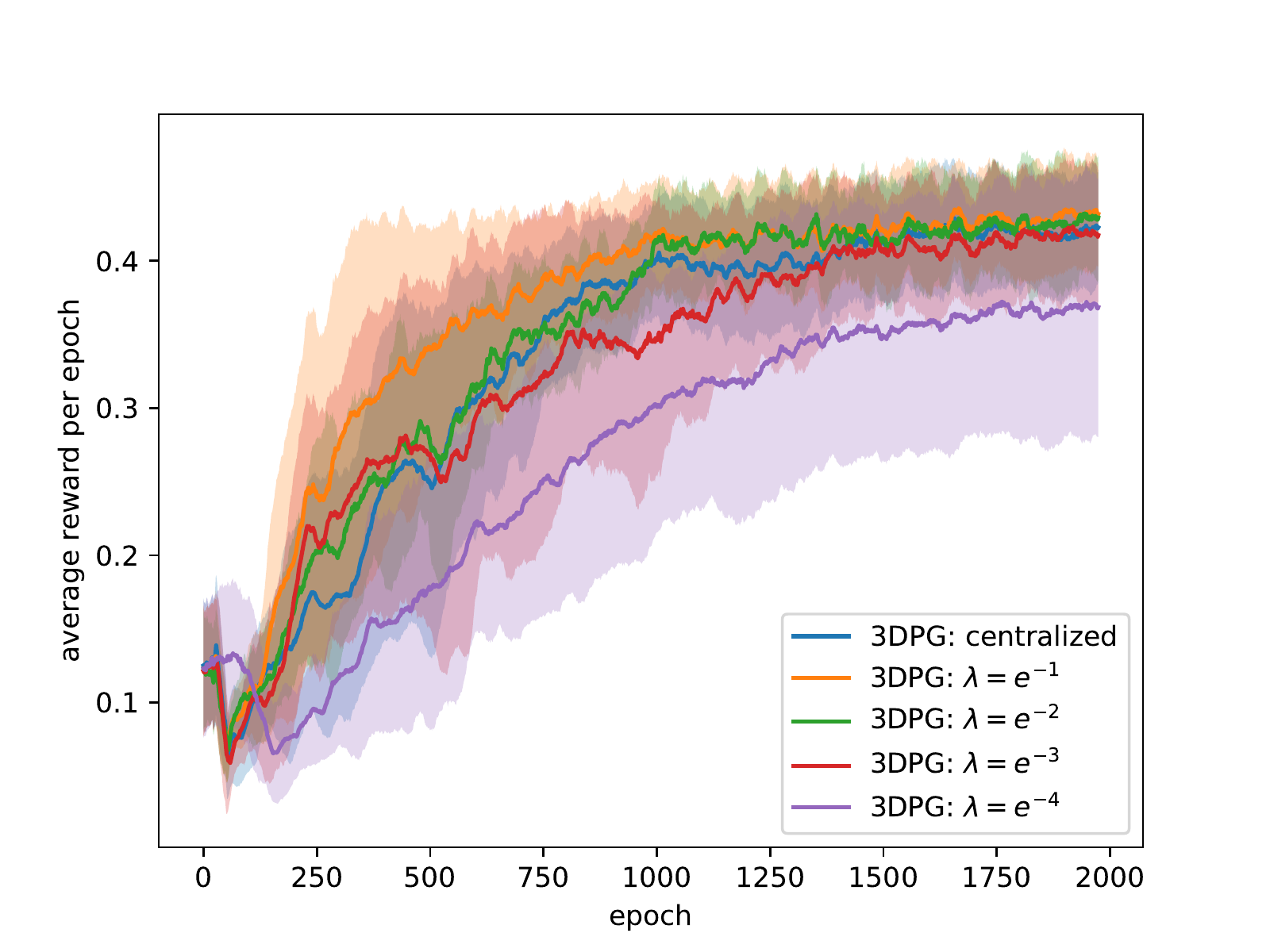}}%
\qquad
\subfigure[Snapshot of the experienced AoI at agent~$1$]{%
\label{fig:aoi}%
\includegraphics[width=.47 \textwidth]{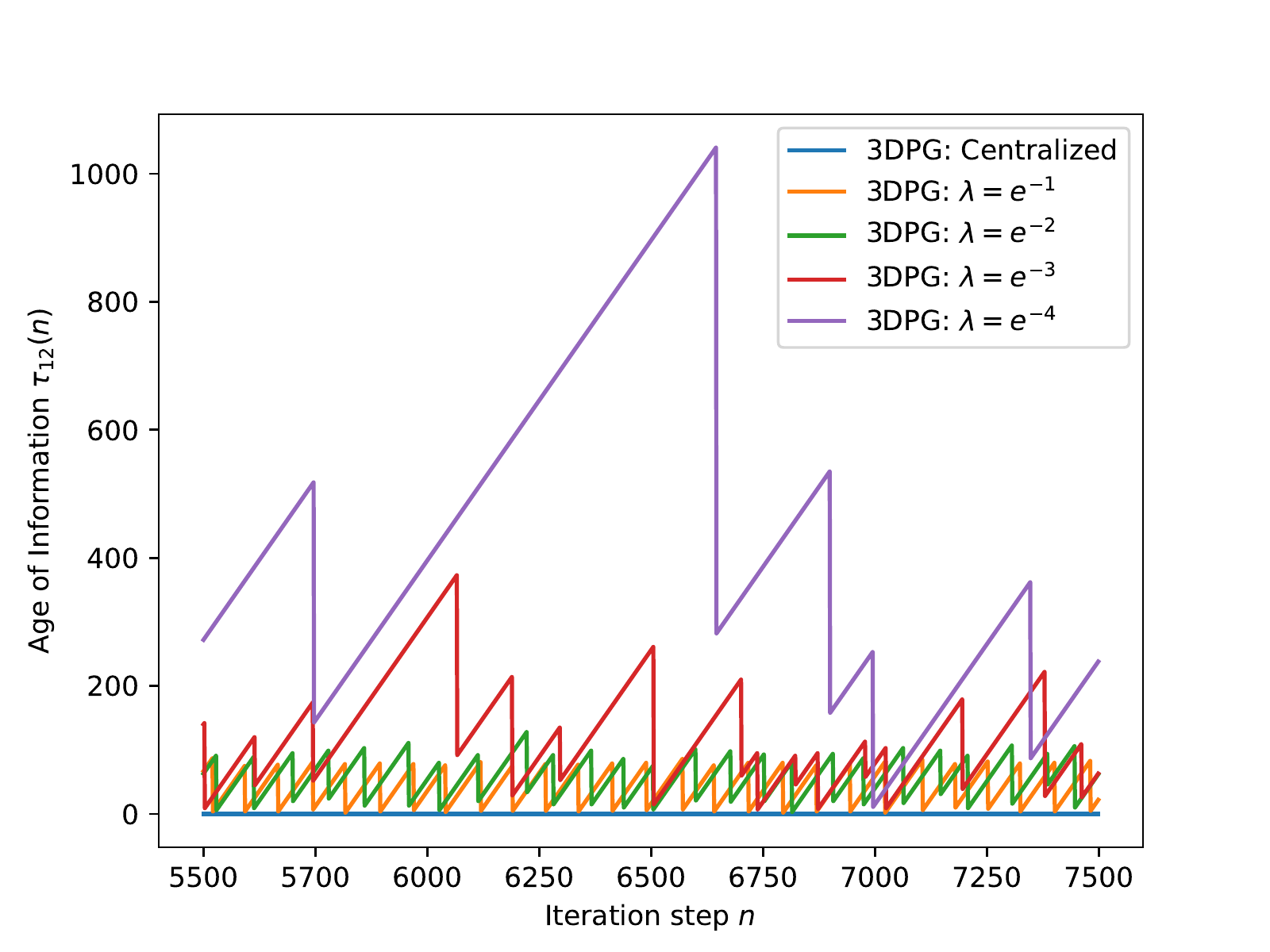}}%
\caption{Comparison of 3DPG with communication; $\lambda$ is the communication success probability.}
\end{figure}

\section{Conclusions and Future Work}
\label{sec:discussions}

In this paper, we presented and analyzed 3DPG, a multi-agent reinforcement learning algorithm for decentralized, online learning in networked systems. We showed that our analysis can be modified to understand the popular MADDPG algorithm \cite{lowe2017multi}. 
Our analysis and numerical examples show that 3DPG should be preferred when a multi-agent decision making problem requires coordinated decisions or a high degree of exploration.

In addition, we presented the first set of data availability assumptions (A1) for distributed online actor-critic learning in network systems. The assumptions describe how old locally available global data tuples are allowed to be so that 3DPG  converges. Our numerical experiments show that 3DPG is highly robust to the use of this old information, which makes it attractive for distributed online multi-agent learning.

For future work, we plan to analyze the trade off between using network resources for policy communication vs. using network resources for data communication. In addition, we are working on an analysis of 3DPG an its variants under a traditional two-timescale step-size schedule.

\bibliographystyle{IEEEtran}
\bibliography{references}

\appendices

\section{Missing proofs}
\label{app:proofs}

\begin{proof}[Proof of \Cref{lem: local_lip}]
    By (A5) the neural network activations are twice-continuous differentiability ($C^2$), hence $\pi^i(s; \phi^i)$ and $Q^i(s,a; \theta^i)$ are $C^2$ in their input coordinates. Additionally, it follows from  \cite[Lemma 9]{ramaswamy2021deep} that $\pi^i(s^i; \phi^i)$ and $Q^i(s,a; \theta^i)$ are $C^2$ in their parameter coordinates $\phi^i$ and $\theta^i$, respectively, for every fixed $s \in \cS$ and $a \in \cA$. Note that composition, product and sums of $C^2$ functions are $C^2$. Moreover $C^2$ functions have local Lipschitz gradients. This is because the gradient is $C^1$, and $C^1$ functions are locally Lipschitz \cite{conway2019course}.
    This immediately shows that $\nabla_{\theta^i} l^i(\theta^i_n, \phi_{n}, s_n, a_n, s_{n+1})$ and $g(\theta^i_n, \phi_{n}, s_n)$ have the required properties.
    
    For $\nabla_{\theta^i} \hat{l}^i(\theta^i_n, \phi_{n}, s_n, a_n)$, fix parameter vectors $\phi$ and $\theta^i$ as well as $s \in \cS$ and $a\in \cA$. Since, $\pi(s; \phi)$ and $Q^i(s,a; \theta^i)$ are $C^2$ in every coordinate, there is some $R>0$ and continuous functions $L_{Q^i}(y,\theta^i, \phi)$ and $L_{\pi}(y,\phi)$, such that $\forall \ \phi_1, \phi_2 \in \overline{B}_R(\phi)$, we have
    \begin{equation}
    \begin{split}
        &\Big\lvert \int Q^i(y, \pi(y; \phi_1) ; \theta^i) p(dy\mid s, a) \\ &\qquad - \int Q^i(y, \pi(y; \phi_2) ; \theta^i) p(dy\mid s, a) \Big\rvert \\
        &\le \int L_{Q^i}(y,\theta^i) \norm{\pi(y; \phi_1) - \pi(y; \phi_2)}_2 p(dy\mid s, a) \\
        &\le \norm{\phi_1 - \phi_2}_2 \int L_{Q^i}(y,\theta^i, \phi) L_{\pi}(y,\phi) p(dy\mid s, a) \\ &\le L(\phi) \norm{ \phi_1 - \phi_2}_2
    \end{split}
    \end{equation}
    for some $L(\phi)>0$. The last inequality follows from the stability of the critic iteration (A3)(a) and the compactness of the state space (A3)(b).
    Hence,  $\nabla_{\theta^i} \hat{l}^i(\theta^i, \phi, s, a)$ is locally Lipschitz as a product and sum of locally Lipschitz functions. It is left to show that $\nabla_{\theta^i} \hat{l}^i(\theta^i, \phi, s, a)$ is continuous in the $s$ and $a$ coordinate. This directly follows from the convergence in distribution by continuity of $p(dy\mid s, a)$ and $r^i(s,a)$, (A4) and (A6) respectively, and since $Q^i(s, \pi(s; \phi); \theta^i)$
    is a bounded continuous function using (A3).
\end{proof}

\begin{proof}[Proof of \Cref{lem:finite_CDFsum}]
	Fix $\varepsilon \in (0,1)$. By (A1)(a) there is a non-negative integer-valued random variable $\overline{\tau}$, such that
	\begin{equation}
		\Pr{\tau_{ij}(n) > \varepsilon n^{\frac{1}{q_1}}} \le \Pr{\overline{\tau}> \varepsilon n^{\frac{1}{q_1}}}
	\end{equation}
	for all $n \in \N_0$ and $\Ew{\overline{\tau}^{q_1}} < \infty$.
	Hence, 
	\begin{align}
		\sum_{n=0}^{\infty} \Pr{\tau_{ij}(n) > \varepsilon n^{\frac{1}{q_1}}} &\le \sum_{m=0}^{\infty} \sum_{n \in \cN(m)}\Pr{\overline{\tau} > \varepsilon n^{\frac{1}{q_1}}} 
		\\ &\le \sum_{m=0}^{\infty} \sum_{n \in \cN(m)}\Pr{\overline{\tau} > m} \\ &= \sum_{m=0}^{\infty} \lvert \cN(m) \rvert \Pr{\overline{\tau} > m},
	\end{align}
	where the sets $\cN(m)$ are defined as
	\begin{equation}
		\cN(m) \coloneqq \{n \in \N_0 : m \le \varepsilon n^{\frac{1}{q_1}} < m+1\} 
	\end{equation}
	for every $m \in \N_0$. The second inequality then follows from the monotonicity of the cumulative distribution function (CDF) by definition of the sets $\cN(m)$.
	Since $\lvert \cN(m) \rvert \le \frac{1}{\varepsilon^{q_1}}\left((m+1)^{q_1}- m^{q_1}\right)$, we have therefore shown that
	\begin{equation}
		\label{eq:sumCDF^ineq}
		\begin{split}
		    \sum_{n=0}^{\infty} \Pr{\tau_{ij}(n) > \varepsilon n^{\frac{1}{q_1}}} &\le \frac{1}{\varepsilon^{q_1}}\sum_{n=0}^{\infty}\lvert \cN(m) \rvert \Pr{\overline{\tau}   >  n} \\ & \le \frac{1}{\varepsilon^{q_1}} \Ew{\overline{\tau}^{q_1}} < \infty.
		\end{split}
	\end{equation}
	The last inequality follows since $\overline{\tau}$ is a non-negative integer-valued random variable, using the following proposition:
	\begin{proposition}
	\label{prop:moment_eq}
	Suppose $X$ is a non-negative integer-valued random variable, then for every $q>0$:
	\begin{equation}
		\Ew{X^q} = \sum_{m=0}^{\infty} ((m+1)^q-m^q)\Pr{X > m}.
	\end{equation}
    \end{proposition}
    
\end{proof}

\begin{proof}{of Lemma \ref{lem:Martingale}: \quad}
    We have 
    \begin{equation}
    \begin{split}
         &\psi^i_n = \gamma \Big(  Q^i(s_{n+1}, \pi(s_{n+1}; \phi_n); \theta_n^i)- \\&\int Q^i(s, \pi(s; \phi_n); \theta_n^i) p(ds \mid s_n, a_n, \phi_n) \Big) \nabla_{\theta^i} Q^i(s_n,a_n; \theta_n^i).
    \end{split}
    \end{equation}
    Define the filtration $\cF_{n-1} \coloneqq \sigma(s_m,a_m,\theta_m, \phi_m \mid m \le n)$ for $n\ge 1$. It then follows that $\{\Psi_n\}$ is a zero-mean martingale. It follows from (A3) and the $C^2$ condition in (A5) that $\sup_{n\ge 0} \norm{\psi^i_n} \le K < \infty$ for a sample path dependent constant $K.$ It then follows from the martingale convergence theorem \cite{durrett2019probability} that $\Psi^i_n$ converges, since $\sum_{m=0}^n \alpha^2(m) \norm{\psi^i_m}^2 < \infty $ almost surely by (A2)(a).
\end{proof}

\begin{proof}{of Lemma \ref{lem: ode_solution}: \quad}
Consider the sequence $\theta^i_n$. The proof for the other parameter sequences are identical. Fix $T>0$. We need to show that
\begin{equation}
\begin{split}
\label{eq: lem3_obj}
    \sup\limits_{t\in[0,T]} \lVert \theta^i_n(t) - \theta^i_\infty(0) - \int_0^t \tilde{\nabla} l^i (\theta_n^i(x), \phi_\infty(x), \mu_\infty(x)) dx \lVert
\end{split}
\end{equation}
converges to zero. The norm in \eqref{eq: lem3_obj} is bounded by
\begin{equation}
\begin{split}
    \lVert \theta^i_n(0) - \theta^i_\infty(0) \lVert   &+ \lVert \int_0^t \tilde{\nabla} l^i (\theta_n^i(x), \phi_n(x), \mu_n(x)) \\ &-  \tilde{\nabla} l^i (\theta_\infty^i(x), \phi_\infty(x), \mu_\infty(x)) dx \lVert. 
\end{split}
\end{equation}
We can now expand the second term, by successively adding zeros for each policy of each agent $j \not= i$. We can then use \Cref{lem: local_lip} to bound the resulting expanison by a term
\begin{equation}
\label{eq: lem3_firstcomp}
    \cO\left( \int_0^t \norm{\theta_n^i(x) - \theta_\infty^i(x) } + \sum_{j \not= i} \norm{\phi_n^j(x) - \phi_\infty^j(x) } dx \right),
\end{equation}
Additionally, we are left with one term of the form
\begin{equation}
\label{eq: lem3_secondcomp}
    \begin{split}
        &\lVert \int_0^t \tilde{\nabla} l^i (\theta_\infty^i(x), \phi_\infty(x), \mu_n(x))\\ &\qquad -  \tilde{\nabla} l^i (\theta_\infty^i(x), \phi_\infty(x), \mu_\infty(x)) dx \lVert. 
    \end{split}
\end{equation}
Due to the compact convergence of every parameter sequences (Arzela-Ascoli theorem) every parameter sequence will converge uniformly over $[0,T]$. This shows that \eqref{eq: lem3_firstcomp} converges to zero. Finally, \eqref{eq: lem3_secondcomp} converges to zero as $\mu_n \rightarrow \mu_\infty$ in distribution and since $\theta_\infty^i(t)$, $\phi_\infty(t)$ are bounded almost surely by (A3)(a) and \Cref{lem: StochApproximation_lemma}.
\end{proof}

\section{Simulation Details}
\label{app:env&algo}

\SetKwComment{Comment}{/* }{ */}
\begin{algorithm}
    \caption{3DPG Algorithm at agent~$i$}
    \label{algo:3DPG}
    Randomly initialize critic and actor weights $\theta^i_0, \phi^i_0$ \;
    Randomly initialize actor weights $\phi^j_0$ for all $j\not = i$ \;
	Initialize replay memory $R^i_0$ and noise process $\cN^i$. \;
    \For{the entire duration}{
        Receive current state $s_n^i$ \;
        Execute action $a_n^i = \mu^i(s^i_n; \phi_n^i ) + \cN_n^i $\;
        Observe $r^i_{n+1}$ and $s^i_{n+1}$ \;
        Allocate local data $(s_n^i, a_n^i, s_{n+1}^i)$ and current local policy $\phi^n_{i}$ for transmission to other agents \;
		Run communication protocols \;
		Store completely received global tuples $t_m^i$ in $R^i_n$ \;
		Sample $M$ transitions from $R_n^i$ \;
		Apply iteration \eqref{eq: AC_iteration} using the sampled transitions \;
   }
\end{algorithm}

For our experiments, we consider a simplified version of the simple spread multi particle coordination problem in \cite{lowe2017multi}. 
Agents and landmarks are represented by points in $[-1,1]^2$. Moreover, agents can move around by choosing a displacement from the set $[-0.1,0.1]^2$. Agents can observe their relative distance to the landmarks and other agents. The actual simple spread environment considers that agents and landmarks take room in space, and the agents are penalized for collisions. 


For our experiments we use target networks for the local policy and critic as well as an Ornstein–Uhlenbeck processes for exploration, both are described in \cite{lillicrap2015continuous}. Both MADDPG and 3DPG use the following algorithm configurations, chosen based on a rough hyperparameter sweep for both algorithms.
\begin{itemize}
    \item Discount factor $\alpha = 0.9$; Replay memory size 20000; Minibatch size 128; Two layer GELU neural networks for each local policy with 64 and 8 neurons and tanh output layer; Two layer GELU neural networks for each local critic with 1024 and 64 neurons.
    \item $\alpha(n) = \frac{e^{-6}}{\frac{n}{1000} + 1}$, $\beta(n) = \frac{e^{-6}}{\frac{n}{1000} + 1} + \frac{e^{-6}}{(\frac{n}{1000} + 1)^2}$
\end{itemize}



\end{document}